%% file: neurips_2021.tex
\newif\ifarxiv
\arxivtrue
\ifarxiv
\documentclass[11pt]{article}
\usepackage[paper=letterpaper,margin=1in]{geometry}
\usepackage[numbers,sort,square]{natbib}
\else
\documentclass{article}

\PassOptionsToPackage{numbers, compress}{natbib}

\usepackage[final]{neurips_2021}
\fi

\usepackage[utf8]{inputenc} %
\usepackage[T1]{fontenc}    %
\usepackage[colorlinks,
            linkcolor=red,
            anchorcolor=blue,
            citecolor=blue
            ]{hyperref}
\usepackage{url}            %
\usepackage{booktabs}       %
\usepackage{amsfonts}       %
\usepackage{nicefrac}       %
\usepackage{microtype}      %
\usepackage[dvipsnames]{xcolor} 

\input{lin_preamble}
\newcommand{\dmix}{\cN^{\textnormal{mix}}_{\sigma,1}}
\newcommand{\dmixN}{\cN^{\textnormal{mix}}_{\sigma,\mu}}
\newcommand{\dmi}{\cN^{\textnormal{mix}}}
\newcommand{\xte}{x_{\textnormal{test}}}
\newcommand{\Eps}{\mathcal{E}}
\newcommand{\lb}{L^{\textnormal{exp}}}

\title{Multiple Descent: Design Your Own Generalization Curve}

\ifarxiv
\title{Multiple Descent: Design Your Own Generalization Curve}
\author{Lin Chen\thanks{Simons Institute for the Theory of Computing, University of California, Berkeley. E-mail: \texttt{lin.chen@berkeley.edu}.} \and 
Yifei Min\thanks{Department of Statistics and Data Science, Yale University. E-mail: \texttt{yifei.min@yale.edu}.} \and 
Mikhail Belkin\thanks{
Hal\i{}c\i{}o\u{g}lu
Data Science Institute, University of California, San Diego. E-mail: \texttt{mbelkin@ucsd.edu}.} \and
Amin Karbasi\thanks{Department of Electrical Engineering, Computer Science, Statistics and Data Science, Yale University. E-mail: \texttt{amin.karbasi@yale.edu}.}}
\date{}
\else

\author{%
  Lin Chen \\
  Simons Institute for the Theory of Computing\\
  University of California, Berkeley\\
  CA 94720 \\
  \texttt{lin.chen@berkeley.edu} \\
  \And
  Yifei Min\\
  Department of Statistics and Data Science \\
  Yale University\\
  CT 06511 \\
  \texttt{yifei.min@yale.edu} \\
  \And
  Mikhail Belkin\\
  Hal\i{}c\i{}o\u{g}lu Data Science Institute \\
  University of California, San Diego\\
  CA 92093 \\
  \texttt{mbelkin@ucsd.edu} \\
  \And
  Amin Karbasi\\
  School of Engineering and Applied Science \\
  Yale University\\
  CT 06511 \\
  \texttt{amin.karbasi@yale.edu} \\
}
\fi

\begin{document}

\maketitle

\begin{abstract}
This paper explores the generalization loss of linear regression in variably parameterized families of models, both under-parameterized and over-parameterized. We show that the generalization curve can have an arbitrary number of peaks, and moreover, locations of those peaks can be explicitly controlled. 
Our results highlight the fact that both classical U-shaped generalization curve and the recently observed double descent curve are not intrinsic properties of the model family. 
Instead, their emergence is due to the interaction between the properties of the data and the inductive biases of learning algorithms.  
\end{abstract}

\section{Introduction}

The main goal of machine learning methods is to provide an accurate out-of-sample prediction, known as generalization. For a fixed family of models, a common way to select a model from this family is through  empirical risk minimization, i.e., algorithmically selecting models that minimize the risk on the training dataset. Given a variably parameterized family of models, the statistical learning theory aims to identify the dependence between model complexity and model performance.
The empirical risk  usually  decreases monotonically as the model complexity increases, and achieves its minimum when the model is rich enough to interpolate the training data, resulting in  zero (or near-zero) training error. In contrast, the behaviour of the test error as a function of model complexity is far more complicated. Indeed, in this paper we show how to construct a model family for which the generalization curve can be fully controlled  (away from the interpolation threshold) in both under-parameterized and over-parameterized regimes.  
Classical statistical learning theory supports a U-shaped curve of generalization versus model complexity \citep{geman1992neural,hastie2009elements}. Under such a framework, the best model is found at the bottom of the U-shaped curve, which corresponds to  appropriately balancing under-fitting and over-fitting the training data. From the view of the bias-variance trade-off, a higher model complexity increases the variance while decreasing the bias. A model with an appropriate level of complexity achieves a relatively low bias while still keeping the variance under control. On the other hand, a model that interpolates the training data is deemed to over-fit and tends to 
worsen the generalization performance due to the soaring variance.  

Although classical statistical  theory suggests a pattern of  behavior for the generalization curve up to the interpolation threshold, it does not describe what happens beyond the interpolation threshold, commonly referred to as  the over-parameterized regime. This is the exact regime where many modern machine learning models, especially deep neural networks, achieved remarkable success. Indeed, neural networks generalize well even when the models are so complex that they have the potential to interpolate all the training data points \citep{zhang2016understanding, belkin2018understand, ghorbani2019linearized, hastie2019surprises}.

Modern practitioners commonly deploy deep neural networks with hundreds of millions or even billions  of parameters.  It has become widely accepted that large models achieve performance superior to small models that may be suggested 
 by the classical U-shaped generalization curve \citep{bengio2003neural, krizhevsky2012imagenet, szegedy2015going, he2016deep, huang2019gpipe}. 
This indicates that the test error decreases again once model complexity grows beyond the interpolation threshold, resulting in the so called double-descent phenomenon described in \citep{belkin2018reconciling}, which has been broadly supported by empirical evidence \citep{neyshabur2014search, neal2018modern, geiger2019jamming, geiger2020scaling} and  confirmed empirically on modern neural architectures by \citet{nakkiran2019deep}. On the theoretical side, this phenomenon has been recently addressed by several works on various model settings. In particular, \citet{belkin2019two} proved the existence of double-descent phenomenon for linear regression with random feature selection and analyzed the random Fourier feature model \citep{rahimi2008random}. \citet{mei2019generalization} also studied the Fourier model and computed the asymptotic test error which captures the double-descent phenomenon. \citet{bartlett2020benign,tsigler2020benign} analyzed and gave explicit conditions for ``benign overfitting'' in linear and ridge regression, respectively. \citet{caron2020finite} provided a finite sample analysis of the nonlinear function estimation and showed that the parameter learned through empirical risk minimization converges to the true parameter with high probability as the model complexity tends to infinity, implying the existence of double descent. \citet{liu2021kernel} studied the high dimensional kernel ridge regression in the under- and over-parameterized regimes and showed that the risk curve can be double descent, bell-shaped, and monotonically decreasing.

Among all the aforementioned efforts, one particularly interesting question is whether one can observe more than two descents in the generalization curve. \citet{d2020triple} empirically showed a sample-wise triple-descent phenomenon under the random Fourier feature model. Similar triple-descent was also observed for linear regression \citep{nakkiran2020optimal}. More rigorously, \citet{liang2019multiple} presented an upper bound on the risk of the minimum-norm interpolation versus the data dimension in Reproducing Kernel Hilbert Spaces (RKHS), which exhibits multiple descent. However, a multiple-descent upper bound without a properly matching lower bound does not imply the existence of a multiple-descent generalization curve.
In this work, we study the multiple descent phenomenon by addressing the following questions:
\begin{compactitem}
    \item Can the existence of a multiple descent generalization curve be rigorously proven?
    \item Can an arbitrary number of descents occur?
    \item Can the generalization curve and the locations of descents be designed?
\end{compactitem}
In this paper, we show that the answer to all three of these questions is yes. 
Further related work is presented in \cref{sec:further-related}.

\textbf{Our Contribution.}
We consider the linear regression model and analyze how the risk changes as the dimension of the data grows. In the linear regression setting, the data dimension is equal to the dimension of the parameter space, which reflects the model complexity. 
We rigorously show that the multiple descent generalization curve exists under this setting. To our best knowledge, this is the first work proving a multiple descent phenomenon. 

Our analysis considers both the underparametrized  and overparametrized regimes. 
In the overparametrized regime, we show that one can control where a descent or an ascent occurs in the generalization curve. This is realized through our algorithmic construction of a feature-revealing process. To be more specific, we assume that the data is in $\mathbb{R}^D$, where $D$ can be arbitrarily large or even essentially infinite. We view each dimension of the data as a feature. We consider a linear regression problem restricted on the first $d$ features, where $d<D$. New features are revealed by increasing the dimension of the data.  We then show that by specifying the distribution of the newly revealed feature to be either a standard Gaussian or a Gaussian mixture, one can determine where an ascent or a descent occurs. In order to create an ascent when a new feature is revealed, it is sufficient that the feature follows a Gaussian mixture distribution. In order to have a descent, it is sufficient that the new feature follows a standard Gaussian distribution. Therefore, in the overparametrized regime, we can fully control the occurrence of a descent and an ascent. As a comparison, in the underparametrized regime, the generalization loss always increases regardless of the feature distribution. 
Generally speaking, we show that we are able to design the generalization curve.

On the one hand, we show theoretically that the generalization curve is malleable and can be constructed in an arbitrary fashion. On the other hand, we rarely observe complex generalization curves in practice, besides carefully curated constructions. Putting these facts together, we arrive at the conclusion that realistic generalization curves arise from specific interactions between properties of typical data  and the inductive biases of algorithms. 
We should highlight that the nature of these interactions is far from being understood and should be an area of  further investigations.

\section{Related Work}\label{sec:further-related}
Our work is directly related to the recent line of research in the  theoretical understanding of the double descent \citep{belkin2019two, hastie2019surprises, xu2019number, mei2019generalization} and the multiple descent phenomenon \citep{liang2019multiple,li2021minimum}. Here we briefly discuss some other work that is closely related to this paper.

\paragraph{Least Square Regression.} 
In this paper we focus on the least square linear regression with no regularization. For the regularized least square regression, \citet{de2005model} proposed a selection procedure for the regularization parameter. \citet{advani2017high} analyzed the generalization of neural networks with mean squared error under the asymptotic regime where both the sample size and model complexity tend to infinity. \citet{richards2020asymptotics} proved for least square regression in the asymptotic regime that as the dimension-to-sample-size ratio $d/n$ grows, an additional peak can occur in both the variance and bias due to the covariance structure of the features. As a comparison, in this paper the sample size is fixed and the model complexity increases. \citet{rudi2017generalization} studied kernel ridge regression and gave an upper bound on the number of the random features to reach certain risk level. Our result shows that there exists a natural setting where by manipulating the random features one can control the risk curve.

\paragraph{Over-Parameterization and Interpolation.} 
The double descent occurs when the model complexity reaches and increases beyond the interpolation threshold. Most previous works focused on proving an upper bound or optimal rate for the risk. %
\citet{caponnetto2007optimal} gave the optimal rate for least square ridge regression via careful selection of the regularization parameter. \citet{belkin2019does} showed that the optimal rate for risk can be achieved by a model that interpolates the training data. In a series of work on kernel regression with regularization parameter tending to zero (a.k.a.\ kernel \emph{ridgeless} regression), \citet{rakhlin2019consistency} showed that the risk is bounded away from zero when the data dimension is fixed with respect to the sample size.  \citet{liang2018just} then considered the case when $d \asymp n$, showed empirically the multiple descent phenomenon and proved a risk upper bound that can be small given favorable data and kernel assumptions. Instead of giving a bound, our paper presents an exact computation of risk in the cases of underparametrized and overparametrized linear regression, and proves the existence of the multiple descent phenomenon. \citet{wyner2017explaining} analyzed AdaBoost and Random Forest from the perspective of interpolation. There has also been a line of work on wide neural networks \citep{arora2019exact, arora2019fine,arora2019harnessing,du2019gradient,allen2019convergence,wei2019regularization,cao2019generalization,advani2020high, chen2020deep,zou2020gradient,song2021convergence}. 

\paragraph{Sample-wise Double Descent and Non-monotonicity.} There has also been recent development beyond the  model-complexity double-descent phenomenon. For example, regarding sample-wise non-monotonicity, \citet{nakkiran2019deep} empirically observed the epoch-wise double-descent and sample-wise non-monotonicity for neural networks.  \citet{chen2020more} and \citet{min2020curious} identified and proved the sample-wise double descent under the adversarial training setting, and \citet{javanmard2020precise} discovered double-descent under adversarially robust linear regression. \citet{loog2019minimizers} showed that empirical risk minimization can lead to sample-wise non-monotonicity in the standard linear model setting under various loss functions including the absolute loss and the squared loss, which covers the range from classification to regression. We also refer the reader to their discussion of the earlier work on non-monotonicity of generalization curves. 
\citet{dar2020subspace} demonstrated the double descent curve of the generalization errors of subspace fitting problems.  \citet{fei2020risk} studied the risk-sample tradeoff in reinforcement learning.

\section{Preliminaries and Problem Formulation}\label{sec:prelim}

\textbf{Notation.}
For $x\in \bR^D$ and $d\le D$, we let $x[1:d]\in \bR^d$ denote a $d$-dimensional vector with $x[1:d]_i=x_i$ for all $1\le i\le d$. For a matrix $A\in \bR^{n\times d}$, we denote its Moore-Penrose pseudoinverse by $A^+\in \bR^{d\times n}$ and denote its spectral norm by $\|A\|\triangleq \sup_{x\ne 0}\frac{\|Ax\|_2}{\|x\|_2}$, where $\|\cdot\|_2$ is the Euclidean norm for vectors. If $v$ is a vector, its spectral norm $\|v\|$ agrees with the Euclidean norm $\|v\|_2$. Therefore, we write $\|v\|$ for $\|v\|_2$ to simplify the notation. We use the big O notation $\cO$ and write variables in the subscript of $\cO$ if the implicit constant depends on them. For example, $\cO_{n,d,\sigma}(1)$ is a constant that only depends on $n$, $d$, and $\sigma$. If $f(\sigma)$ and $g(\sigma)$ are functions of $\sigma$, write $f(\sigma)\sim g(\sigma)$ if $\lim \frac{f(\sigma)}{g(\sigma)} = 1$. It will be given in the context how we take the limit. 

\textbf{Distributions.}
Let $\cN(\mu,\sigma^2)$ ($\mu,\sigma\in \bR$) and $\cN(\mu,\Sigma)$ ($\mu\in \bR^n$, $\Sigma\in \bR^{n\times n}$) denote the univariate and multivariate Gaussian distributions, respectively, where $\mu\in \bR^n$ and $\Sigma\in \bR^{n\times n}$ is a positive semi-definite matrix. 
We define a family of \textit{trimodal} Gaussian mixture distributions as follows \begin{equation*}
\dmixN \triangleq{} \frac{1}{3}\cN(0,\sigma^2)+\frac{1}{3}\cN(-\mu,\sigma^2)+\frac{1}{3}\cN(\mu,\sigma^2)\,.
\end{equation*}
For an illustration, please see \cref{fig:contour}.

\begin{figure*}[htb]
    \centering
	\begin{subfigure}[b]{0.32\textwidth}
		\includegraphics[width=\textwidth]{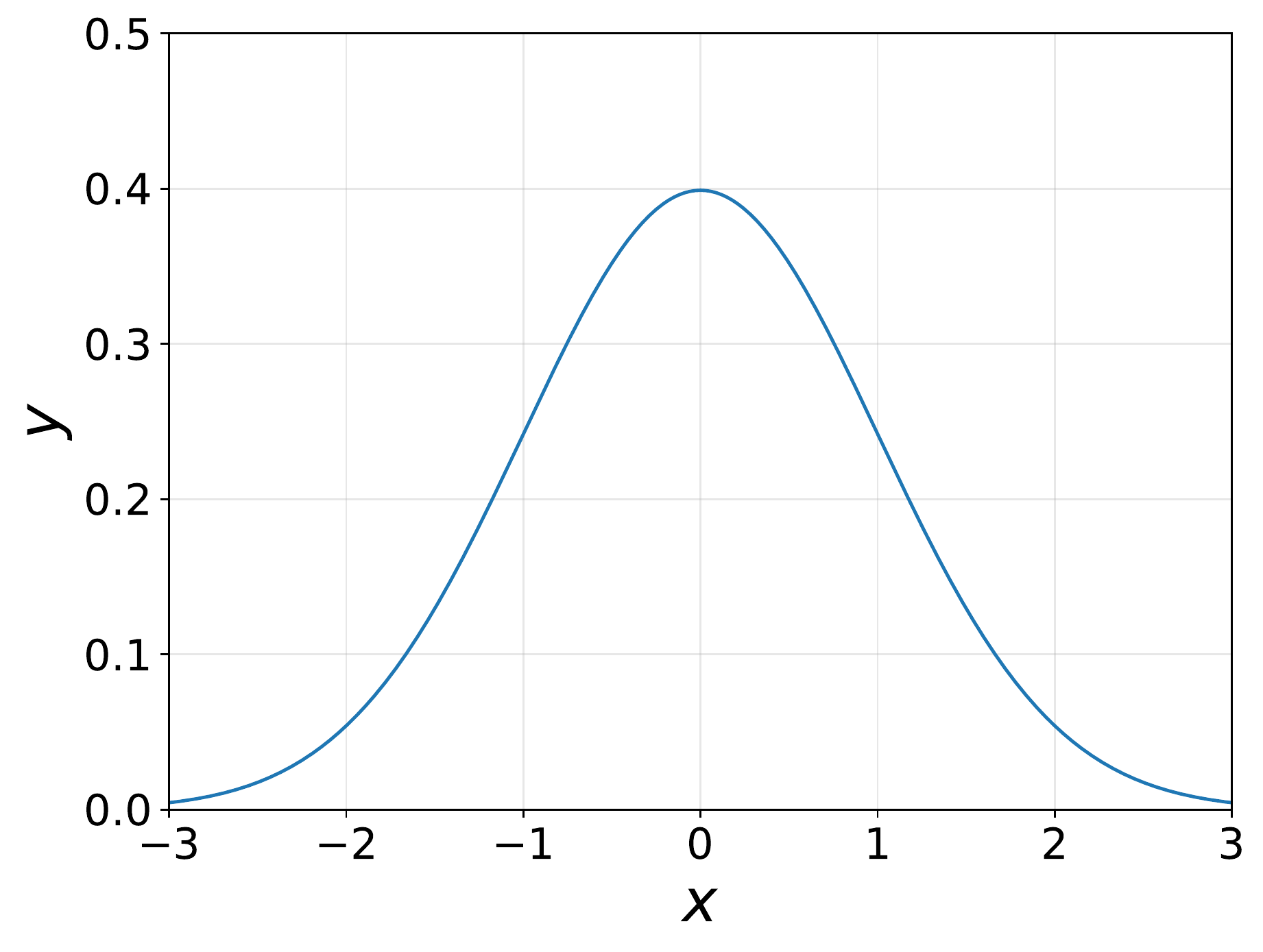}
		\caption{$\cN(0,1)$ feature}
		\label{fig:gaussian_density}
	\end{subfigure}
	\begin{subfigure}[b]{0.32\textwidth}
		\includegraphics[width=\textwidth]{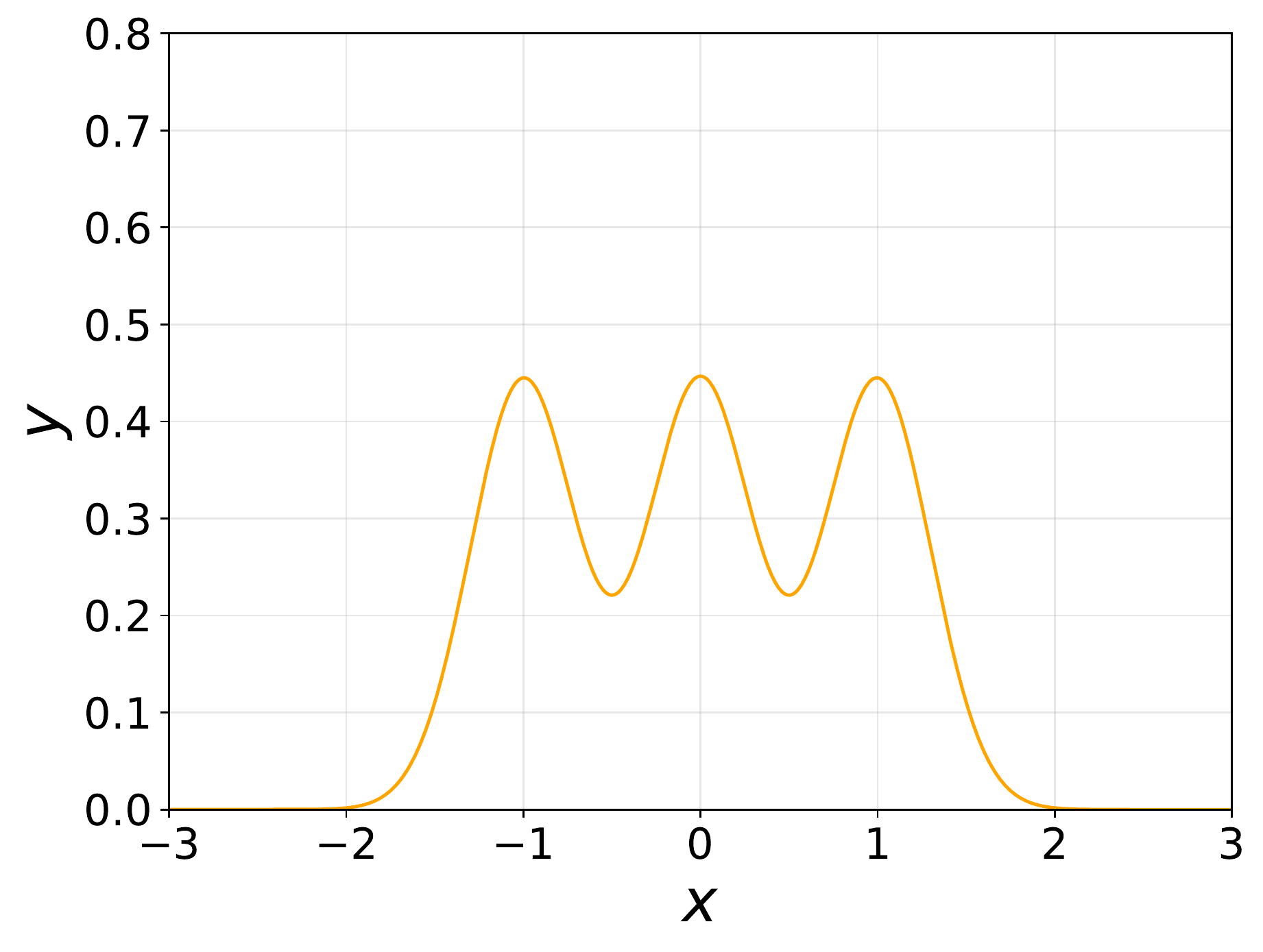}
		\caption{$\dmix$ feature, ($\sigma = 0.3$)}
		\label{fig:gau_mix_density1}
	\end{subfigure}
	\begin{subfigure}[b]{0.32\textwidth}
		\includegraphics[width=\textwidth]{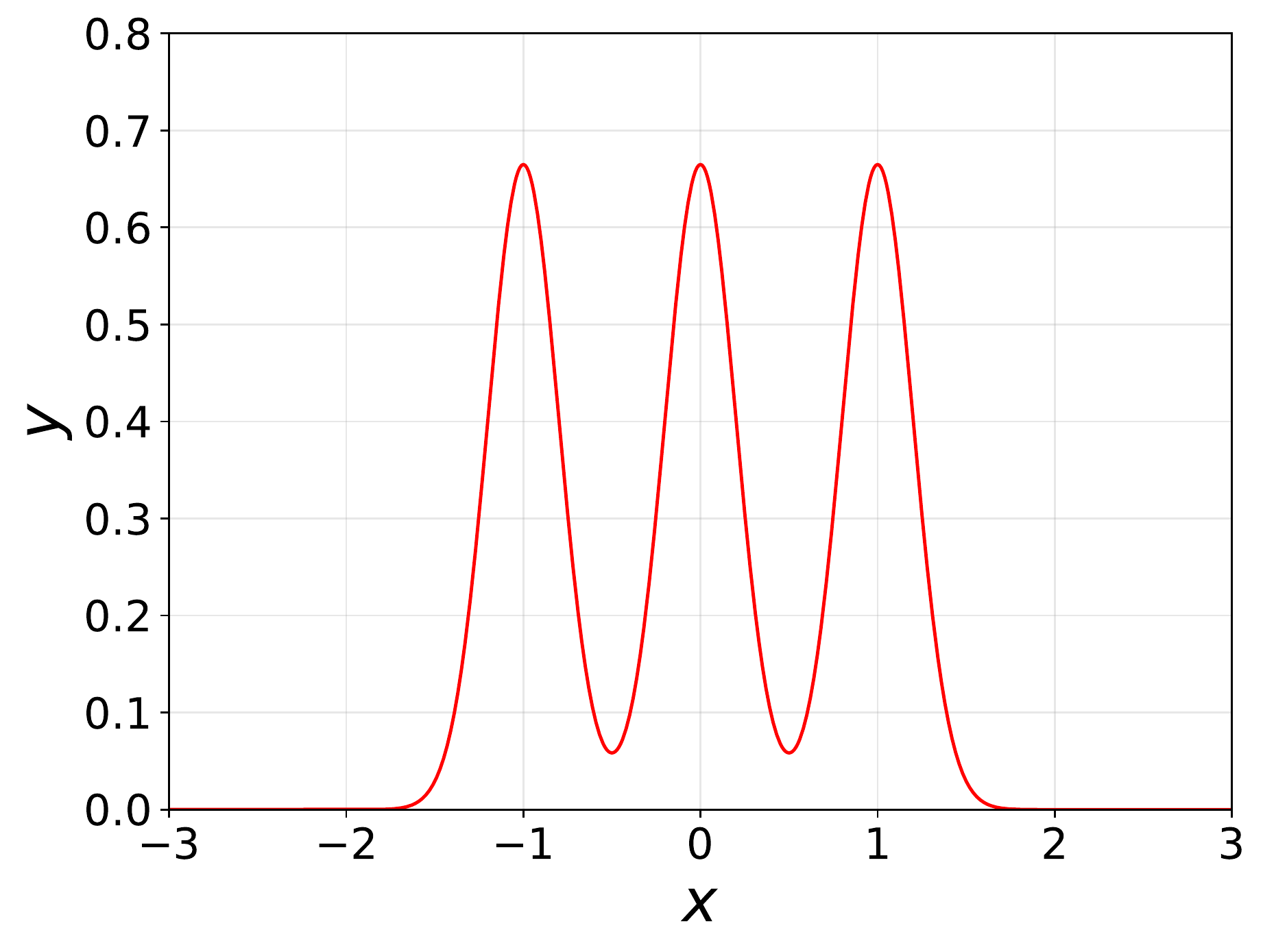}
		\caption{$\dmix$ feature, ($\sigma = 0.2$)}
		\label{fig:gau_mix_density2}
	\end{subfigure}
	\caption{Density functions of the $\cN(0,1)$ and $\cN_{\sigma,1}^{\textnormal{mix}}$ feature. A new entry is independently sampled from the 1-dimensional distribution being either a standard Gaussian or trimodal Gaussian mixture. Smaller $\sigma$ leads to higher concentration around each modes.}
	\label{fig:contour}
\end{figure*}

Let $\chi^2(k,\lambda)$ denote the noncentral chi-squared distribution with $k$ degrees of freedom and   the non-centrality parameter $\lambda$. For example, if $X_i\sim\cN(\mu_i,1)$ (for $i=1,2,\dots,k$) are independent Gaussian random variables, we have $\sum_{i=1}^k X_i^2\sim \chi^2(k,\lambda)$, where $\lambda=\sum_{i=1}^k \mu_i^2$. We also denote by $\chi^2(k)$ the (central) chi-squared distribution  with $k$ degrees and the $F$-distribution by $F(d_1,d_2)$ where $d_1$ and $d_2$ are the degrees of freedom. 

\textbf{Problem Setup.}
Let $x_1,\dots,x_n\in \bR^D$ be column vectors that represent the training data of size $n$ and let $\xte\in \bR^D$ be a column vector that represents the test data. We assume that they are all independently drawn from a distribution\[
x_1,\dots,x_n,\xte\stackrel{iid}{\sim}\cD\,.
\]

Let us consider a linear regression problem on the first $d$ features, where $d\le D$ for some arbitrary large $D$. Here, $d$ can be viewed as the number of features revealed. Then the feature vectors are $\tilde{x}_1,\dots,\tilde{x}_n$, where  $\tilde{x}_i = x_i[1:d]\in \bR^d$ denotes the first $d$ entries of $x_i$. The corresponding response variable $y_i$ satisfies
\[y_i = \tilde{x}_i^\T \beta +\eps_i\,,\quad i=1,\dots,n\,, \]
where the noise $\eps_i\sim \cN(0,\eta^2)$. We use the same setup as in \citep{hastie2019surprises} (see Equations (1) and (2) in \citep{hastie2019surprises}). Moreover, in another closely related work  \citep{liang2019multiple}, if the kernel is set to the linear kernel,  it is equivalent to our setup.

Next, we introduce the estimate $\hat{\beta}$ of $\beta$ and its excess generalization loss. Let $\eps = (\eps_1,\dots,\eps_n)^\T \in \bR^n$ denote the noise vector. The \emph{design} matrix $A$ equals $[\tilde{x}_1,\dots,\tilde{x}_n]^\top\in \bR^{n\times d}$. 
 Let $x=\xte[1:d]$ denote the first $d$ features of the test data. 
For the underparametrized regime where $d<n$, the least square solution on the training data is $A^+(A\beta + \eps)$. For the overparametrized regime where $d > n$, $A^+(A\beta + \eps)$ is the minimum-norm solution. In both regimes we consider the solution $ \hat{\beta} \triangleq A^+(A\beta + \eps)$. The excess generalization loss on the test data is then given by \begin{align}
    L_d \triangleq{}& \bE\left[ \left(y-x^\T \hat{\beta} \right)^2 - \left( y-x^\T \beta \right)^2 \right]\nonumber\\
={}& \bE\left[ \left( x^\T (\hat{\beta}- \beta) \right)^2  \right]\nonumber\\
={}& \bE\left[ \left( x^\T \left( (A^+ A-I)\beta + A^+ \eps \right) \right)^2 \right]\nonumber\\
={}& \bE\left[(x^\T (A^+ A - I)\beta)^2\right] + \bE\left[ (x^\T A^+ \eps)^2 \right]\nonumber\\
={}& \bE\left[(x^\T (A^+ A - I)\beta)^2\right] + \eta^2 \bE\left\| (A^\T)^+ x\right\|^2\,,\label{eq:total_loss}
\end{align}
where $y = x^\T \beta + \eps_\textnormal{test} $ and $ \eps_\textnormal{test} \sim \cN(0,\eta^2)$. 
We call the term $\bE\left[(x^\T (A^+ A - I)\beta)^2\right]$ the \emph{bias} and call the term $\eta^2 \bE\left\| (A^\T)^+ x\right\|^2$ the \emph{variance}. 

The next remark shows that in the underparametrized regime, the bias vanishes. The vanishing bias in the underparametrized regime is also observed by \citet{hastie2019surprises} and shown in their Proposition~2.
\begin{remark}\label{rmk:underparam-bias-vanishes}
In the underparametrized regime, if $\cD$ is a continous distribution (our construction presented later satisfies this condition), the matrix $A$ has independent column almost surely. In this case, we have $A^+ A = I$ and therefore the bias $\bE\left[(x^\T (A^+ A - I)\beta)^2\right]$ vanishes irrespective of $\beta$. In other words, in the underparametrized regime, $L_d$ equals $\eta^2 \bE\| (A^\top)^+ x \|^2$. 
\end{remark}

According to \cref{rmk:underparam-bias-vanishes}, we have $L_d = \eta^2 \bE\| (A^\top)^+ x \|^2$ in the underparametrized regime. It also holds in the overparametrized regime when $\beta=0$. Without loss of generality, we assume $\eta=1$ in the underparametrized regime (for all $\beta$). In the overparametrized regime, we also assume $\eta=1$ for the $\beta=0$ case. In this case, we have \begin{equation}\label{eq:loss-var}
      L_d = \bE\| (A^\top)^+ x \|^2\,.
 \end{equation}
 We assume a general $\eta$ (i.e., not necessarily being $1$) in the overparametrized regime when $\beta$ is non-zero. 
 
We would like to study the change in the loss caused by the growth in the number of features revealed. Recall $L_d = \bE\|(A^\top)^+ x\|^2$. Once we reveal a new feature, which adds a new row $b^\top$ to $A^\top$ and a new component $a_1$ to $x$, we have $L_{d+1} = \bE \left\|\begin{bmatrix}
A^\top \\
b^\top
\end{bmatrix}^+  \begin{bmatrix}
x\\
a_1
\end{bmatrix}\right\|^2$.

\textbf{Local Maximum and Multiple Descent.}
Throughout the paper, we say that a local maximum occurs at a dimension $d\geq 1$ if $L_{d-1}<L_d$ and  $L_d>L_{d+1}$.
Intuitively, a local maximum occurs if there is an increasing stage of the generalization loss, followed by a decreasing stage, as the dimension $d$ grows. 
Additionally, we define $L_0\triangleq -\infty$.
If the generalization loss exhibits a single descent, based on our definition, a unique local maximum occurs at $d=1$. For a double-descent generalization curve, a local maximum occurs at two different dimensions. In general, if we observe local maxima at multiple dimensions, we say there is a multiple descent. 

\section{Underparametrized Regime}\label{sec:underparam}

First, we present our main theorem for the underparametrized regime below, whose proof is deferred to the end of \cref{sec:underparam}. It states that the generalization loss $L_d$ is always non-decreasing as $d$ grows. Moreover, it is possible to have an arbitrarily large ascent, i.e., $L_{d+1} - L_d>C$ for any $C >0$. 

\begin{theorem}[Proof in \cref{sec:unnormalized-underparam-proof}]\label{thm:unnormalized-underparam}
If $d<n$, we have $L_{d+1}\ge L_d$ irrespective of the data distribution. Moreover, for any $C>0$, there exists a distribution $\cD$ such that $L_{d+1} - L_d > C$. 
\end{theorem}

\begin{remark}[$\cD$ can be a product distribution]\label{rmk:product-distribution}
The first part of \cref{thm:unnormalized-underparam} holds irrespective of the data distribution. For the second part of the theorem ( i.e., for any $C>0$ there exists a distribution such that $L_{d+1} - L_d > C$) to hold,  one extremely simple and elegant choice of the distribution $\cD$ is a product distribution $\cD=\cD_1\times \dots \times \cD_D$ such that $x_{i,j}\stackrel{iid}{\sim} \cD_j$ for all $1\le i\le n$, where $\cD_j$ is a Gaussian mixture $\dmi_{\sigma_j,1}$ for some $\sigma_j>0$. 
Since the second part of \cref{thm:unnormalized-underparam} is of independent interest, the result is summarized by \cref{thm:ascent}.
\end{remark}

\begin{remark}[Kernel regression on Gaussian data]
In light of \cref{rmk:product-distribution}, $\cD$ can be chosen to be a product distribution that consists $\dmi_{\sigma_j}$. Note that one can simulate $\dmix$ with $\cN(0,1)$ through the inverse transform sampling. To see this, let $F_{\cN(0,1)}$ and $F_{\dmix}$ be the cdf of $\cN(0,1)$ and $\dmix$, respectively. If $X\sim \cN(0,1)$, we have $F_{\cN(0,1)}(X)\sim \unif((0,1))$ and therefore $\varphi_\sigma(X)\triangleq F^{-1}_{\dmix}(F_{\cN(0,1)}(X))\sim \dmix$. In fact, we can use a multivariate Gaussian $\cD'=\cN(0,I_{D\times D})$ and a sequence of non-linear kernels $k^{[1:d]}(x,x')\triangleq \langle\phi^{[1:d]}(x),\phi^{[1:d]}(x')\rangle$, where the feature map is $\phi^{[1:d]}(x)\triangleq [\phi_1(x_1),\phi_2(x_2),\dots,\phi_d(x_d)]^\top\in \bR^d$. Here is a simple rule for defining $\phi_j$: if $\cD_j=\dmi_{\sigma_j}$, we set $\phi_j$ to $\varphi_{\sigma_j}$. Thus, the problem becomes a kernel regression problem on the standard Gaussian data.
\end{remark}

The first part of \cref{thm:unnormalized-underparam}, which says that $L_d$ is increasing (or more precisely, non-decreasing), agrees with Figure 1 of
\citep{belkin2019two} and Proposition 2 of \citep{hastie2019surprises}. In \citep{hastie2019surprises}, they proved that the risk increases with
$\gamma = d/n$. Note that, at first glance, \cref{thm:unnormalized-underparam} may look counterintuitive since it does not obey the classical U-shaped generalization curve. However, we would like to emphasize that the U-shaped curve does not always occur.
In Figure 1 and Proposition 2 of these two papers respectively, there is no U-shaped curve. The intuition behind \cref{thm:unnormalized-underparam} is that in the underparametrized setting, the bias is always zero and as $d$ approaches $n$, the variance keeps increasing. 

Coming to the second part of Theorem~\ref{thm:unnormalized-underparam}, we now discuss how we will construct such a distribution $\cD$ inductively to satisfy $L_{d+1} - L_d > C$.
We fix $d$. Again, denote the first $d$ features of $\xte$ by $x\triangleq \xte[1:d]$.  
Let us add an additional component to the training data $x_1[1:d],\dots,x_n[1:d]$ and test data $x$ so that the dimension $d$ is incremented by 1. Let $b_i\in \bR$ denote the additional component that we add to the vector $x_i$ (so that the new vector is given as $[x_i[1:d]^\top,b_i]^\top$. Similarly, let $a_1\in \bR$ denote the additional component that we add to the test vector $x$. We form the column vector $b=[b_1,\dots,b_n]^\top \in \bR^n$ that collects all additional components that we add to the training data. 

We consider the change in the generalization loss 
as follows
\begin{equation}\label{eq:change_in_loss}
\begin{split}
    L_{d+1}-L_d ={}&
\bE\left[\left\|  \begin{bmatrix}
A^\top \\
b^\top
\end{bmatrix}^+ \begin{bmatrix}
x\\
a_1
\end{bmatrix}\right\|^2 -  \left\|  (A^+)^\top x \right\|^2\right]\,.
\end{split}
\end{equation}

Note that the components $b_1,\dots,b_n,a_1$ are i.i.d. The proof of \cref{thm:unnormalized-underparam} starts with \cref{lem:pseudo-inverse} which relates the pseudo-inverse of $[A,b]^\top$ to that of $A^\top$. In this way, we can decompose $\left\|  \begin{bmatrix}
A^\top \\
b^\top
\end{bmatrix}^+ \begin{bmatrix}
x\\
a_1
\end{bmatrix}\right\|^2$ into multiple terms for further careful analysis in the proofs hereinafter.
\begin{lemma}[Proof in \cref{sec:proof-pseudo-inverse}]\label{lem:pseudo-inverse}
    Let $A\in \bR^{n\times d}$ and $0\ne b\in \bR^{n\times 1}$, where $n\ge d+1$. Additionally, let $P=AA^+$ and $Q=bb^+ = \frac{bb^\top}{\|b\|^2}$, and define $z\triangleq \frac{b^\top(I-P)b}{\|b\|^2}$. If $z\ne 0$ and the columnwise partitioned matrix $[A,b]$ has linearly independent columns, we have 
    \begin{align*}
        \begin{bmatrix}
            A^\top\\
            b^\top
        \end{bmatrix}^+
        ={}&
        \begin{bmatrix}
        \left(I - \frac{bb^\top}{\|b\|^2}\right)\left(I+\frac{AA^+bb^\top}{\|b\|^2-b^\top AA^+b}\right)(A^+)^\top, \frac{(I-AA^+)b}{\|b\|^2-b^\top AA^+b}
        \end{bmatrix}\\
        ={}& \begin{bmatrix}
        (I-Q)(I+\frac{PQ}{1-\tr(PQ)})(A^+)^\top, \frac{(I-P)b}{b^\top(I-P)b}
        \end{bmatrix}\\
        ={}& \begin{bmatrix}
        (I-Q)(I+\frac{PQ}{z})(A^+)^\top, \frac{(I-P)b}{b^\top(I-P)b}
        \end{bmatrix}\,.
    \end{align*}
\end{lemma}

In our construction of $\cD$, the components $\cD_j$ are all continuous distributions. The matrix $I-P$ is an orthogonal projection matrix and therefore $\rank(I-P) =n-d$. As a result, it holds almost surely that $b\ne 0$, $z\ne 0$, and $[A,b]$ has linearly independent columns. Thus the assumptions of \cref{lem:pseudo-inverse} are satisfied almost surely. In the sequel, we assume that these assumptions are always fulfilled.

\cref{thm:descent} guarantees that if $L_d = \bE\left\|  (A^+)^\top x \right\|^2 $ is finite and the $(d+1)$-th features $b_1,\dots,b_n,a_1$ are i.i.d.\ sampled from $\cN(0,1)$ or $\dmix$, $L_{d+1} = \bE\left\|  \begin{bmatrix}
A^\top \\
b^\top
\end{bmatrix}^+ \begin{bmatrix}
x\\
a_1
\end{bmatrix}\right\|^2 $ is also finite.

\begin{theorem}[Proof in \cref{sec:proof-descent}]\label{thm:descent}
Let $z$ be as defined in \cref{lem:pseudo-inverse}. If $b_1,\dots,b_n,a_1$ are i.i.d.\ and follow a distribution with mean zero, conditioned on $A$ and $x$, we have
\[
 \bE_{b,a_1}\left[\left\|  \begin{bmatrix}
A^\top \\
b^\top
\end{bmatrix}^+ \begin{bmatrix}
x\\
a_1
\end{bmatrix}\right\|^2 \right]
\le{} \bE_{b,a_1}\left[\frac{1}{z} \left\|  (A^+)^\top x \right\|^2 + \frac{a_1^2}{b^\top (I-P)b}\right]\,.
\]
In particular, if $d+2<n$ and $b_1,\dots,b_n,a_1\stackrel{iid}{\sim}\cN(0,1)$, conditioned on $A$ and $x$, we have \[
\bE_{b,a_1}\left[\left\|  \begin{bmatrix}
A^\top \\
b^\top
\end{bmatrix}^+ \begin{bmatrix}
x\\
a_1
\end{bmatrix}\right\|^2 \right] \le \frac{(n-2)\left\|  (A^+)^\top x \right\|^2 +1}{n-d-2}\,.
\] If $d+2<n$ and $b_1,\dots,b_n,a_1\stackrel{iid}{\sim}\dmix$, conditioned on $A$ and $x$, we have
\[\bE_{b,a_1}\left\|  \begin{bmatrix}
A^\top \\
b^\top
\end{bmatrix}^+ \begin{bmatrix}
x\\
a_1
\end{bmatrix}\right\|^2 \le \frac{(n-2+\sqrt{d})\left\|  (A^+)^\top x \right\|^2+ 2/(3\sigma^2)+1}{n-d-2}
\,.
\]
\end{theorem}

Using \cref{thm:descent}, we can show inductively (on $d$) that $L_d$ is finite for every $d$. Provided that we are able to guarantee finite $L_1$, \cref{thm:descent} implies that $L_d$ is finite for every $d$ if the components are always sampled from $\cN(0,1)$ or $\dmix$.

Making a large $L_d$ can be achieved by adding an entry sampled from $\dmix$ when the data dimension increases from $d-1$ to $d$ in the previous step. 
\cref{thm:ascent} shows that adding a $\dmix$ feature can increase the loss by arbitrary amount, which in turn implies the second part of \cref{thm:unnormalized-underparam}.

\begin{theorem}[Proof in \cref{sec:proof-ascent}]\label{thm:ascent}
For any $C > 0$ and $\bE\left\|(A^+)^\top x\right\|^2<+\infty$,  
there exists a $\sigma>0$ such that if $b_1,\dots,b_n,a_1\stackrel{iid}{\sim}\dmix$, we have \begin{align*}
\bE\left[\left\|  \begin{bmatrix}
A^\top \\
b^\top
\end{bmatrix}^+ \begin{bmatrix}
x\\
a_1
\end{bmatrix}\right\|^2 -  \left\|  (A^+)^\top x \right\|^2\right]>{}& C\,.
\end{align*}
\end{theorem}

We are now ready to prove \cref{thm:unnormalized-underparam}.

\subsection{Proof of \cref{thm:unnormalized-underparam}}\label{sec:unnormalized-underparam-proof}

\begin{proof}
We follow the notation convention in \eqref{eq:change_in_loss}:
\[
L_{d+1}-L_d ={}
\bE\left[\left\|  \begin{bmatrix}
A^\top \\
b^\top
\end{bmatrix}^+ \begin{bmatrix}
x\\
a_1
\end{bmatrix}\right\|^2 -  \left\|  (A^\T)^+ x \right\|^2\right]\,.
\]
Recall $d<n$ and the matrix $B'\triangleq \begin{bmatrix}
A^\top \\
b^\top
\end{bmatrix}$ is of size $(d+1)\times n$. Both matrices $B'$ and $B\triangleq A^\T$ are fat matrices. As a result, if $x' \triangleq \begin{bmatrix}
x\\
a_1
\end{bmatrix}$,  we have \[
\|B'^+ x'\|^2 = \min_{z:B'z=x'} \|z\|^2\,,\quad \|B^+ x\|^2 = \min_{z:Bz=x} \|z\|^2\,.
\]
Since $ \{z\mid B'z=x'\} \subseteq \{z\mid Bz=x\}$, we get $ \|B'^+ x'\|^2 \ge \|B^+ x\|^2$. Therefore, we obtain $L_{d+1}\ge L_d$. 
The second part follows from \cref{thm:ascent}. 
\end{proof}

\begin{remark}\label{rmk:order-of-revealing}
\cref{rmk:product-distribution} and the proof of \cref{thm:ascent} indicate that $\cD = \cD_1 \times \cdots \times \cD_D$ is a product distribution. The construction in the proof also shows that the generalization curve is determined by the specific choice of the $\cD_i$'s. Note that  permuting the order of $\cD_i$'s is equivalent to changing the order by which the features are being revealed (i.e., permuting the entries of the data $x_i$'s). Therefore,  given the same data points $x_1,\cdots,x_n \in \mathbb{R}^D$, one can create different generalization curves simply by changing the order of the feature-revealing process. 
\end{remark}

\section{Overparametrized Regime}\label{sec:overparam}

\begin{figure*}
	\centering
	\includegraphics[width=0.55\linewidth]{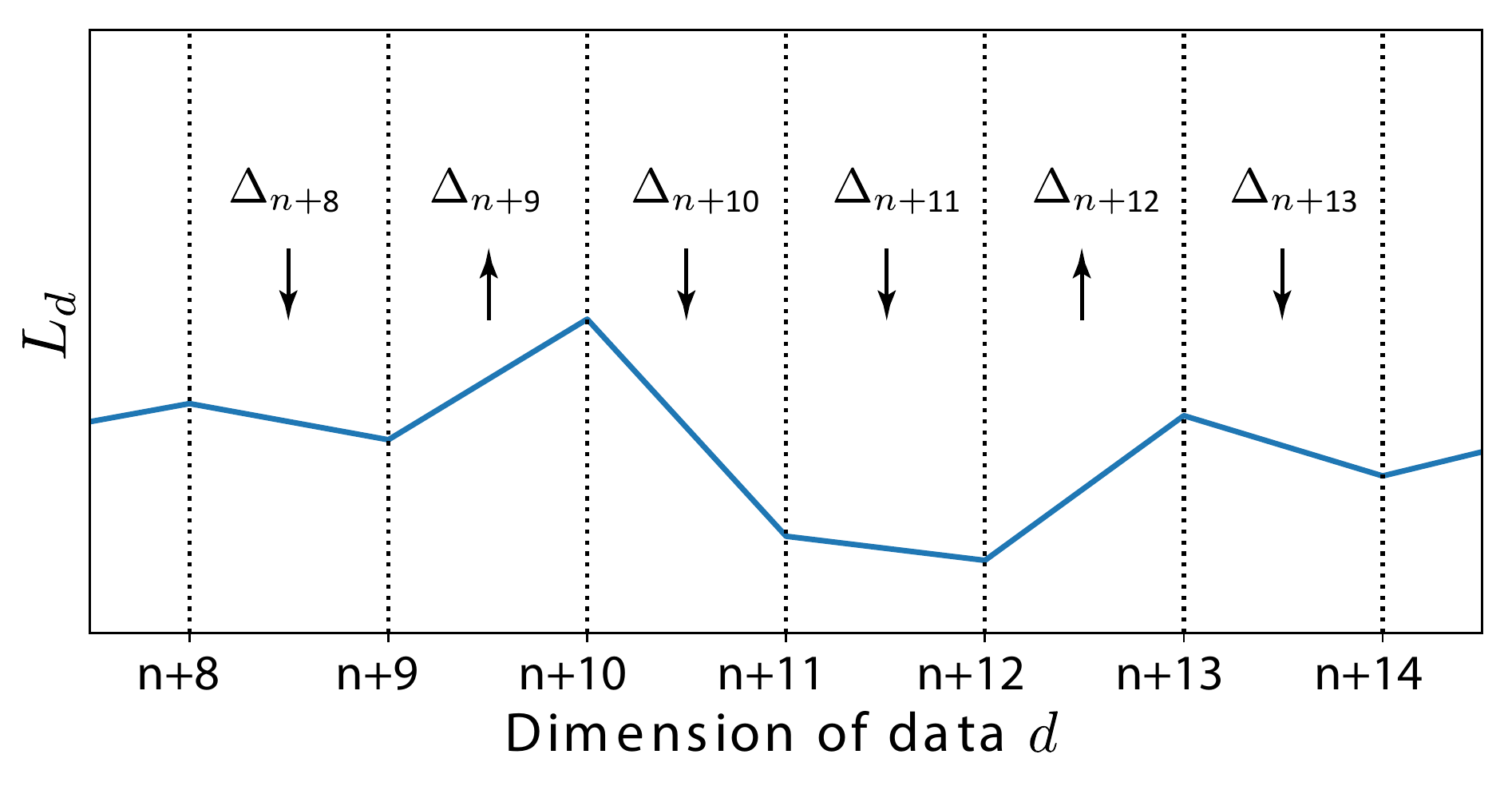}
	\caption{Illustration of the multiple descent phenomenon for the generalization loss $L_d$  versus the dimension of data $d$ in the \emph{overparametrized} regime starting from $d=n+8$. One can fully control the generalization curve to increase or decrease as specified by the sequence $\Delta=\{ {\downarrow}, {\uparrow}, {\downarrow}, {\downarrow}, {\uparrow}, {\downarrow}, \dots \}$.
	Adding a new feature with Gaussian mixture distribution increases the loss, while adding one with Gaussian distribution decreases the loss.}
	\label{fig:cartoon_overparam}
\end{figure*}

In this section, we study the multiple decent phenomenon in the overparametrized regime. Note that as stated in \cref{sec:prelim}, we consider the minimum-norm solution here. We first consider the case where the model $\beta=0$ and $L_d$ is as defined in \eqref{eq:loss-var}. Then we discuss the setting $\beta\ne 0$. 

As stated in the following theorem, we require $d \geq n+8$. This is merely a technical requirement and we can still say that $d$ starts at roughly the same order as $n$. In other words, the result covers almost the entire spectrum of the overparametrized regime. 

\begin{theorem}[Overparametrized regime, $\beta=0$]\label{thm:main-overparam}
Let $n < D-9$. Given any sequence $\Delta_{n+8},\Delta_{n+9},\dots$, $\Delta_{D-1}$ where $\Delta_d\in \{{\uparrow}, {\downarrow}\}$, there exists a distribution $\cD$ such that for every $n+8 \le d\le D-1$, we have \[
L_{d+1} \begin{cases}
> L_{d}, \quad\text{if } \Delta_d = {\uparrow}\\
< L_{d}, \quad\text{if } \Delta_d = {\downarrow}\,.
\end{cases}
\] 
\end{theorem}

In \cref{thm:main-overparam}, the sequence $\Delta_{n+8}$, $\Delta_{n+9}$, $\cdots$, $\Delta_{D-1}$ is just used to specify the increasing/decreasing behavior of the $L_d$ sequence for $d>n+8$. Compared to \cref{thm:unnormalized-underparam} for the underparametrized regime, where $L_d$ always increases, \cref{thm:main-overparam} indicates that one is able to fully control both ascents and descents in the overparametrized regime. \cref{fig:cartoon_overparam} is an illustration.

We now present tools for proving \cref{thm:main-overparam}. \cref{lem:pseudo-inverse-overparam} gives the pseudo-inverse of $A$ when $d > n$. 

\begin{lemma}[Proof in \cref{sec:proof-pseudo-inverse-overparam}]\label{lem:pseudo-inverse-overparam}
Let $A\in \bR^{n\times d}$ and $b\in \bR^{n\times 1}$, where $n\le d$. Assume that matrix $A$ and the columnwise partitioned matrix $B\triangleq [A,b]$ have linearly independent rows. Let $G \triangleq (AA^\top)^{-1}\in \bR^{n\times n}$ and $
    u \triangleq \frac{b^\top G}{1+b^\top G b}\in \bR^{1\times n}$. We have 
    \begin{equation*}
        \begin{bmatrix}
            A^\top\\
            b^\top
        \end{bmatrix}^+ =
        \begin{bmatrix}
        (I-bu)^\top (A^+)^\top, u^\top
        \end{bmatrix}\,.
    \end{equation*}
\end{lemma}

\cref{lem:overparam-assumptions-hold} establishes finite expectation for several random variables. These finite expectation results are necessary for \cref{thm:descent-overparam2} and \cref{thm:ascent-overparam} to hold. Technically, they are the dominating random variables needed in Lebesgue's dominated convergence theorem. \cref{lem:overparam-assumptions-hold} indicates that to guarantee these finite expectations, it suffices to set the first $n+8$ distributions to the standard normal distribution and then set $\cD_{n+8},\dots,\cD_D$ to either a Gaussian or a Gaussian mixture distribution. In fact, in \cref{thm:descent-overparam2} and \cref{thm:ascent-overparam}, we always add a Gaussian distribution or a Gaussian mixture. 

\begin{lemma}[Proof in \cref{sec:proof-overparam-assumptions-hold}]\label{lem:overparam-assumptions-hold}
Let $\cD = \cD_1 \times \cdots \times \cD_D$ be a product distribution where \begin{enumerate}[label=(\alph*)]
\item $\cD_d = \cN(0,1)$ if $d = 1,\dots, n+8$; and
\item $\cD_d$ is either $\cN(0,\sigma_d^2)$ or $\cN^\textnormal{mix}_{\sigma_d,\mu_d}$ for $d > n+8$.
\end{enumerate} 
Let $\cD_{[1:d]}$ denote $\cD_1\times \cdots \times \cD_{d}$. 
Assume that every row of $A\in \bR^{n\times d}$ and $x\in \bR^{d\times 1}$ are i.i.d.\ and follow $\cD_{[1:d]}$.  
For any $d$ such that $n+8 \leq d\leq D$, all of the followings hold:
\begin{equation}\label{eq:finite-expectations}
\begin{alignedat}{2}
    \bE&[ \| (A^+)^\top x \|^2] <{} + \infty\,,
     &&\bE[ \lambda^2_{\textnormal{max}} ((AA^\top)^{-1}) ] <{} + \infty\,,\\
     \bE&[\lambda_{\textnormal{max}} ((AA^\top)^{-1}) \| (A^+)^\top x \|^2] <{} +\infty\,,\quad
     &&\bE[\lambda^2_{\textnormal{max}} ((AA^\top)^{-1}) \| (A^+)^\top x \|^2] <{} +\infty\,.
\end{alignedat}
\end{equation}
\end{lemma}

\cref{thm:descent-overparam2,thm:ascent-overparam} are the key technical results for constructing multiple descent in the overparametrized regime. One can create a descent ($L_{d+1}<L_d$) by adding a Gaussian feature (\cref{thm:descent-overparam2}) and create an ascent ($L_{d+1}>L_d$) by adding a Gaussian mixture feature (\cref{thm:ascent-overparam}). 

\begin{theorem}[Proof in \cref{sec:proof-descent-overparam2}]\label{thm:descent-overparam2}
If $ \bE[\|(A^\T A)^+ x\|^2] >0$ and all equations in \eqref{eq:finite-expectations} hold,
there exists $\sigma>0$ such that if $a_1,b_1,\dots,b_n\stackrel{iid}{\sim} \cN(0,\sigma^2)$, we have 
\begin{align*}
L_{d+1} - L_d ={}& \bE\left\| \begin{bmatrix}
A^\top \\
b^\top
\end{bmatrix}^+ \begin{bmatrix}
x\\
a_1
\end{bmatrix}\right\|^2 - \bE \left\|  (A^+)^\top x \right\|^2 < 0\,.
\end{align*}
\end{theorem}

\cref{thm:ascent-overparam} shows that adding a Gaussian mixture feature can make $L_{d+1}>L_d$. 

\begin{theorem}[Proof in \cref{sec:proof-ascent-overparam}]\label{thm:ascent-overparam}
Assume $\bE \| (A^+)^\top x \|^2 < + \infty$. For any $C >0$, there exist $\mu$, $\sigma>0$ such that if $a_1,b_1,\dots,b_n\stackrel{iid}{\sim} \dmixN$, we have
\begin{align*}
L_{d+1} - L_d ={}& \bE\left\| \begin{bmatrix}
A^\top \\
b^\top
\end{bmatrix}^+ \begin{bmatrix}
x\\
a_1
\end{bmatrix}\right\|^2 - \bE \left\|  (A^+)^\top x \right\|^2 >C\,.
\end{align*}
\end{theorem}

The proof of \cref{thm:main-overparam} immediately follows from \cref{thm:descent-overparam2} and \cref{thm:ascent-overparam}.

\begin{proof}[Proof of \cref{thm:main-overparam}]

We construct the product distribution $\cD = \prod_{d=1}^D \cD_d$. We set $\cD_d = \cN(0,1)$ for $d = 1,\dots, n+8$. 
For $n+8<d \leq D$, $\cD_d$ is either $\cN(0,\sigma_d^2)$ or $\cN^\textnormal{mix}_{\sigma_d,\mu_d}$ depending on $\Delta_d$ being either $\downarrow$ or $\uparrow$. 

First we show that for each step $d$, the assumption $ \bE[\|(A^\T A)^+ x\|^2] >0$ of \cref{thm:descent-overparam2}  is satisfied. 
If $ \bE[\|(A^\T A)^+ x\|^2] =0$, we know that $(A^\T A)^+ x=0$ almost surely. 
Since $\cD$ is a continuous distribution, the matrix $A$ has full row rank almost surely. Therefore, $\rank((A^\T A)^{+}) = \rank(A^\T A) = n$ almost surely. Thus $\dim\ker (A^\T A)^+ = d - n \le d - 1$ almost surely, which implies $x\notin \ker (A^\T A)^+$. In other words, $(A^\T A)^+ x\ne 0$ almost surely. We reach a contradiction. 
Moreover, by \cref{lem:overparam-assumptions-hold}, the assumption $\bE \| (A^+)^\top x \|^2 < + \infty$ of \cref{thm:ascent-overparam} is also  satisfied.

If $\Delta_{d-1} = {\downarrow}$, by \cref{thm:descent-overparam2}, there exists $\sigma_{d} >0$ such that if $\cD_{d} = \cN(0,\sigma^2_{d})$, then $L_d < L_{d-1}$. Similarly if $\Delta_{d-1} = {\uparrow}$, by \cref{thm:ascent-overparam}, there exists $\sigma_{d}$ and $\mu_{d}$ such that $\cD_{d} = \cN^\textnormal{mix}_{\sigma_d,\mu_d}$ guarantees $L_d>L_{d-1}$.

\end{proof}

 \textbf{Gaussian $\beta$ setting.}  In what follows, we study the case where the model $\beta$ is non-zero. In particular, we consider a setting where each entry of $\beta$ is i.i.d. $ \cN(0,\rho^2)$.
Recalling \eqref{eq:total_loss}, define the biases \begin{align*}
    \Eps_d \triangleq (x^\T (A^+ A-I)\beta)^2 , \quad
    \Eps_{d+1} \triangleq  \left([x^\T,a_1]([A,b]^+[A,b]-I)\begin{bmatrix}
\beta\\
\beta_1
\end{bmatrix}\right)^2  \,, 
\end{align*}
and the expected risks \begin{equation}\label{eq:bayesian risk}
\begin{split}
    \lb_d \triangleq 
    \E[\Eps_d] + \eta^2 \bE\left\| (A^\T)^+ x\right\|^2 , \quad
    \lb_{d+1} \triangleq   \bE[\Eps_{d+1}]
+ \eta^2 \bE
\left\|  \begin{bmatrix}
A^\top \\
b^\top
\end{bmatrix}^+ \begin{bmatrix}
x\\
a_1
\end{bmatrix}\right\|^2 \,,
\end{split}
\end{equation}
where $\beta\sim \cN(0,\rho^2I_d)$ and $\beta_1\sim \cN(0,\rho^2)$. 
The second term in $\lb_d$ and $\lb_{d+1}$ is the variance term. Note that $\lb_d$ is the expected value of $L_d$ in \eqref{eq:total_loss} and averages over $\beta$. \cref{thm:bias} shows that one can add a Gaussian mixture feature in order to make $\lb_{d+1}>\lb_{d}$, and add a Gaussian feature in order to make $\lb_{d+1}<\lb_{d}$.

\begin{theorem}[Proof in \cref{sec:bias}]\label{thm:bias}
Let $a_1,\beta_1\in \bR$, $x\in \bR^{d\times 1}$,  $\beta\in \bR^{d\times 1}$, $A\in \bR^{n\times d}$ and $b\in \bR^{n\times 1}$, where $n\le d$.  Assume that $x,a_1,\beta_1,\beta,A,b$ are jointly independent, $[\beta^\T,\beta_1]^\T\sim \cN(0,\rho^2 I_{d+1})$. Moreover, assume that the matrix $[A,b]$ has linearly independent rows almost surely. The following statements hold:
\begin{enumerate}[label=(\alph*)]
\item If $a_1,b_1,\dots,b_n\stackrel{iid}{\sim} \dmixN$, for any $C>0$, there exist $\mu,\sigma$ such that $\lb_{d+1}-\lb_d>C$. 
\item If  $a_1,b_1,\dots,b_n\stackrel{iid}{\sim} \cN(0,\sigma^2)$,  
there exists $\sigma>0$ such that for all  \[\rho \le \eta\sqrt{ \frac{ \bE[\|(A^\T A)^+ x\|^2]}{ \E\|A^{+\T}x\|^2 + 1}} \,,\] we have $\lb_{d+1}<\lb_d$. 
\end{enumerate}
\end{theorem}

\cref{thm:bias} indicates that for $\beta$ obeying a normal distribution, one can still construct a generalization curve as desired by adding a Gaussian or Gaussian mixture feature properly. We make this construction explicit for any desired generalization curve  in (the proof of) \cref{thm:main-overparam2}. Similar to the construction in the underparametrized regime (for all $\beta$) and overparametrization regime (for $\beta=0$), the distribution $\cD$ can be made a product distribution. 

\begin{theorem}[Overparametrized regime, $\beta$ being Gaussian]\label{thm:main-overparam2}
Let $n < D-9$. Given any sequence $\Delta_{n+8},\Delta_{n+9},\dots$, $\Delta_{D-1}$ where $\Delta_d\in \{{\uparrow}, {\downarrow}\}$, there exists $\rho>0$ and a distribution $\cD$ such that for $\beta\sim \cN(0,\rho^2)$ and every $n+8 \le d\le D-1$, we have \[
\lb_{d+1} \begin{cases}
> \lb_{d}, \quad\text{if } \Delta_d = {\uparrow}\\
< \lb_{d}, \quad\text{if } \Delta_d = {\downarrow}\,.
\end{cases}
\] 
\end{theorem}

\begin{proof}[Proof of \cref{thm:main-overparam2}]
Define the design matrix $A_d\triangleq [x_1[1:d],\dots,x_n[1:d]]^\top\in \bR^{n\times d}$. 
Similar to the proof of \cref{thm:main-overparam}, we construct the product distribution $\cD = \prod_{d=1}^D \cD_d$. We set $\cD_d = \cN(0,1)$ for $d = 1,\dots, n+8$. 
For $n+8<d \leq D$, $\cD_d$ is either $\cN(0,\sigma_d^2)$ or $\cN^\textnormal{mix}_{\sigma_d,\mu_d}$ depending on $\Delta_d$ being either $\downarrow$ or $\uparrow$. 

If $\Delta_{d-1} = {\uparrow}$, by \cref{thm:bias}, there exists $\sigma_{d}$ and $\mu_{d}$ such that $\cD_{d} = \cN^\textnormal{mix}_{\sigma_d,\mu_d}$ guarantees $\lb_d>\lb_{d-1}$.
If $\Delta_{d-1} = {\downarrow}$, define \[\rho_d\triangleq  \eta \sqrt{\frac{ \bE[\|(A_{d-1}^\T A_{d-1})^+ \xte[1:d-1]\|^2]}{ \E\|A_{d-1}^{+\T}\xte[1:d-1]\|^2 + 1  }}\,.\]
By \cref{thm:bias}, there exists $\sigma_{d} >0$ such that if $\rho\le \rho_d$ and $\cD_{d} = \cN(0,\sigma^2_{d})$, then $\lb_d < \lb_{d-1}$. We take \[\rho = \min_{d:\Delta_{d-1} = {\downarrow}} \rho_d\,.\]
\end{proof}

\section{Conclusion}

Our work proves that the expected risk of linear regression can manifest  multiple descents when the number of features increases and sample size is fixed. This is carried out through an algorithmic construction of a feature-revealing process where the newly revealed feature follows either a Gaussian distribution or a Gaussian mixture distribution. Notably, the construction also enables us to control local maxima in the underparametrized regime and control ascents/descents freely in the overparametrized regime. Overall, this allows us to design the generalization curve away from the interpolation threshold.

We believe that our analysis of linear regression in this paper is a good starting point for explaining non-monotonic generalization curves observed in machine learning studies. Extending these results to more complex problem setups would be a meaningful future direction.

\section*{Funding Transparency Statement}
LC: Funding in direct support of this work: postdoctoral research fellowship by the Simons Institute for the Theory of Computing, University of California, Berkeley, and Google PhD Fellowship by Google. Additional revenues related to this work: internships at Google.

MB acknowledges support from NSF IIS-1815697, and the support of the NSF and the Simons Foundation for the Collaboration on the Theoretical Foundations of Deep Learning through awards DMS-2031883 and \#814639.

AK: Funding in direct support of this work: NSF (IIS-1845032) and ONR (N00014-19-1-2406). 

\bibliographystyle{abbrvnat}
\bibliography{reference-list}

\clearpage

\appendix

\section{Almost Sure Convergence of Sequence of Normal Random Variables}
In this paper, we need a sequence of random variables $\{X_n\}_{n\ge 1}$ such that $X_n\sim \cN(0,\sigma_n^2)$, $\lim_{n\to +\infty}\sigma_n = 0$, and $X_n\to 0$ almost surely. The following lemma shows the existence of such a sequence. 
\begin{lemma}\label{lem:bc}
There exist a sequence of random variables $\{X_n\}_{n\ge 1}$ such that $X_n\sim \cN(0,\sigma_n^2)$, $\lim_{n\to +\infty}\sigma_n = 0$, and $X_n\to 0$ almost surely.
\end{lemma}
\begin{proof}
Let $\sigma_n = 1/n^2$ and $X_n\sim \cN(0,\sigma_n^2)$. Define the event $E_n\triangleq \{ |X_n|>\eps \}$. We have \[
\sum_{n=1}^\infty \bP(E_n) = \sum_{n=1}^\infty \bP(|\cN(0,1)|>\eps/\sigma_n)
\le \sum_{n=1}^\infty \frac{\sigma_n}{\eps} e^{-\frac{\eps^2}{2\sigma_n}}
\le \sum_{n=1}^\infty \frac{\sigma_n}{\eps} = \sum_{n=1}^\infty \frac{1}{\eps n^2} < +\infty\,.
\]
By the Borel–Cantelli lemma, we have $\bP(\limsup_{n\to +\infty} E_n)=0$, which implies that $X_n\to 0$ almost surely. 
\end{proof}

\section{Proofs for Underparametrized Regime}

\subsection{Proof of \cref{lem:pseudo-inverse}}\label{sec:proof-pseudo-inverse}
By \citep[Theorem~1]{baksalary2007particular}, we have \[
\begin{bmatrix}
            A^\top\\
            b^\top
        \end{bmatrix}^+ = \begin{bmatrix}
            (I-Q)A(A^\top(I-Q)A)^{-1}, \frac{(I-P)b}{b^\top (I-P)b)}\,.
        \end{bmatrix}
    \]
    Define $r\triangleq A^\top b\in \bR^{d}$. Since $A$ has linearly independent columns, the Gram matrix $G=A^\top A$ is non-singular. The Sherman-Morrison formula gives \[
    (A^\top(I-Q)A)^{-1} = \left(A^\top A - \frac{rr^\top}{\|b\|^2}\right)^{-1}
    = G^{-1} + \frac{G^{-1}rr^\top G^{-1}}{\|b\|^2 - r^\top G^{-1} r}
    = G^{-1} + \frac{G^{-1}rb^\top (A^+)^\top}{\|b\|^2 - r^\top G^{-1} r}\,,
    \]
    where we use the facts $r=A^\top b$ and $AG^{-1}=(A^+)^\top$ in the last equality. Therefore, we deduce \begin{align*}
    A(A^\top(I-Q)A)^{-1} ={}& AG^{-1} + \frac{AG^{-1}rb^\top (A^+)^\top}{\|b\|^2 - r^\top G^{-1} r} \\
    ={}& (A^+)^\top + \frac{AG^{-1}A^\top bb^\top (A^+)^\top}{\|b\|^2 - r^\top G^{-1} r} \\
    ={}& \left(I + \frac{AA^+ bb^\top }{\|b\|^2 - r^\top G^{-1} r}\right)(A^+)^\top \\
    ={}& \left(I + \frac{PQ }{1 - \frac{r^\top G^{-1} r}{\|b\|^2}}\right)(A^+)^\top\,.
    \end{align*}
    Observe that \[
    1 - \frac{r^\top G^{-1} r}{\|b\|^2} = 1 - \frac{b^\top A(A^\top A)^{-1}A^\top b}{\|b\|^2} = 1 - \frac{b^\top Pb}{\|b\|^2} = z\,.
    \]
    Therefore, we obtain the desired expression. 

\subsection{Proof of \cref{thm:descent}}\label{sec:proof-descent}
First, we rewrite the expression as follows
\begin{equation}\label{eq:expansion}
    \begin{split}
& \left\| \begin{bmatrix}
A^\top \\
b^\top
\end{bmatrix}^+ \begin{bmatrix}
x\\
a_1
\end{bmatrix}\right\|^2 -  \left\| (A^+)^\top x \right\|^2\\
={}& \left\|(I-Q)(I+PQ/z)(A^+)^\top x + \frac{(I-P)b}{b^\top (I-P)b}a_1\right\|^2 - \|(A^+)^\top x\|^2\,,
\end{split}
\end{equation}
where $P, Q, z$ are  defined in \cref{lem:pseudo-inverse}.
Since $a_1$ has mean 0 and is independent of other random variables, so that the cross term vanishes under expectation over $b$ and $a_1$:\[
\bE_{b,a_1}\left[\left\langle (I-Q)(I+PQ/z)(A^+)^\top x, \frac{(I-P)b}{b^\top (I-P)b}a_1 \right\rangle\right] = 0\,,
\]
 where $\langle \cdot, \cdot\rangle$ denotes the inner product. Therefore taking the expectation of \eqref{eq:expansion} over $b$ and $a_1$ yields 

\begin{align}
& \bE_{b,a_1}\left[\left\| \begin{bmatrix}
A^\top \\
b^\top
\end{bmatrix}^+ \begin{bmatrix}
x\\
a_1
\end{bmatrix}\right\|^2 -  \left\| (A^+)^\top x \right\|^2\right]\\
={}&  \bE_{b,a_1}\left[  \|(I-Q)(I+PQ/z)(A^+)^\top x\|^2 - \|(A^+)^\top x\|^2 + \left\| \frac{(I-P)b}{b^\top (I-P)b}a_1 \right\|^2 \right] \\
\end{align}

We simplify the third term. Recall that $I-P=I-AA^+$ is an orthogonal projection matrix and thus idempotent  \begin{equation}\label{eq:third-term}
 \left\| \frac{(I-P)b}{b^\top (I-P)b}a_1 \right\|^2 = \frac{a_1^2}{(b^\top(I-P)b)^2}\|(I-P)b\|^2 = \frac{a_1^2}{b^\top(I-P)b}\,.
\end{equation}

Thus we have \begin{align}
    & \bE_{b,a_1}\left[\left\| \begin{bmatrix}
A^\top \\
b^\top
\end{bmatrix}^+ \begin{bmatrix}
x\\
a_1
\end{bmatrix}\right\|^2 -  \left\| (A^+)^\top x \right\|^2\right]\\
={}& \bE_{b,a_1}\left[ \|(I-Q)(I+PQ/z)(A^+)^\top x\|^2 - \|(A^+)^\top x\|^2 + \frac{a_1^2}{b^\top(I-P)b} \right]\,.\label{eq:three-terms}
\end{align}

We consider the first and second terms. We write $v=(A^+)^\top x$ and define $z=\frac{b^\top(I-P)b}{\|b\|^2}$. The sum of the first and second terms equals \begin{equation}\label{eq:quadratic-form}
    \|(I-Q)(I+PQ/z)v\|^2 - \|v\|^2 = -v^\top M v\,,
\end{equation}
where 
\[
M \triangleq Q - \frac{PQ+QP}{z} + \left(\frac{2}{z}-\frac{1}{z^2}\right)QPQ + \frac{QPQPQ}{z^2}\,.
\]
The rank of $M$ is at most $2$. To see this, we re-write $M$ in the following way \[
M = \left[Q\left( - \frac{P}{z} + \left(\frac{2}{z}-\frac{1}{z^2}\right)PQ + \frac{PQPQ}{z^2}\right)\right] + \left[ - \frac{PQ}{z}\right]\triangleq M_1+M_2\,.
\]
Notice that $\rank(M_1)\le \rank(Q)$, $\rank(M_2)\le \rank(Q)$, and $\rank(Q)=1$.

It follows that $\rank(M)\le \rank(M_1)+\rank(M_2)=2$. The matrix $M$ has at least $n-2$ zero eigenvalues. We claim that $M$ has two non-zero eigenvalues and they are $1-1/z<0$ and $1$. 

Since $$\rank(PQ)\le \rank(Q)=1$$ and $$\tr(PQ) = \frac{b^\top P b}{\|b\|^2}=1-z,$$ thus $PQ$ has a unique non-zero eigenvalue $1-z$. Let $u\ne 0$ denote the corresponding eigenvector such that $PQu=(1-z)u$. Since $u\in \im P$ and $P$ is a projection, we have $Pu=u$. Therefore we can verify that \[
Mu = (1-\frac{1}{z})u\,.
\]
To show that the other non-zero eigenvalue of $M$ is $1$, we compute the trace of $M$
\[
\tr(M) = \tr(Q) - \frac{2\tr(PQ)}{z} + \left(\frac{2}{z}-\frac{1}{z^2}\right)\tr(PQ) + \frac{\tr((PQ)^2)}{z^2} = 2-\frac{1}{z}\,,
\]
where we use the fact that $\tr(Q)=1$, $\tr(PQ)=1-z$, \[\tr((PQ)^2)=\tr\left(\frac{Pbb^\top Pbb^\top}{\|b\|^4}\right)=
\tr\left(\frac{(b^\top Pb)(b^\top Pb)}{\|b\|^4}\right)=(1-z)^2
\,.\] We have shown that $M$ has eigenvalue $1-1/z$ and $M$ has at most two non-zero eigenvalues. Therefore, the other non-zero eigenvalue is $\tr(M)-(1-1/z)=1$.

We are now in a position to upper bound \eqref{eq:quadratic-form}  as follows: 
\[
-v^\top Mv\le -(1-1/z)\|v\|^2
\,.
\]

Putting all three terms of the change in the dimension-normalized generalization loss yields \[
  \bE_{b,a_1}\left[\left\|  \begin{bmatrix}
A^\top \\
b^\top
\end{bmatrix}^+ \begin{bmatrix}
x\\
a_1
\end{bmatrix}\right\|^2 -  \left\|  (A^+)^\top x \right\|^2\right]
\le{} \bE_{b,a_1}\left[-(1-1/z)\|v\|^2 + \frac{a_1^2}{b^\top (I-P)b}\right]\,.
\]
Therefore, we get\[
 \bE_{b,a_1}\left[\left\|  \begin{bmatrix}
A^\top \\
b^\top
\end{bmatrix}^+ \begin{bmatrix}
x\\
a_1
\end{bmatrix}\right\|^2 \right]
\le{} \bE_{b,a_1}\left[\frac{1}{z}\|v\|^2 + \frac{a_1^2}{b^\top (I-P)b}\right]\,.
\]

For $b_1,\dots,b_n,a_1\stackrel{iid}{\sim} \cN(0,1)$, we have $\bE[a_1^2]=1$. Moreover, $b^\top (I-P)b$ follows $\chi^2(n-d)$ a distribution. Thus $\frac{1}{b^\top(I-P)b}$ follows an inverse-chi-squared distribution with mean $\frac{1}{n-d-2}$. Therefore the expectation  $\bE [\frac{a_1^2}{b^\top(I-P)b}] = \frac{1}{n-d-2}$. 

Notice that $1/z$ follows a $1+\frac{d}{n-d}F(d,n-d)$ distribution and thus $\bE[1/z] = 1+\frac{d}{n-d-2}$. 

As a result, we obtain\[
\bE_{b,a_1}\left[\left\|  \begin{bmatrix}
A^\top \\
b^\top
\end{bmatrix}^+ \begin{bmatrix}
x\\
a_1
\end{bmatrix}\right\|^2 \right] \le \frac{(n-2)\|v\|^2 +1}{n-d-2}
\]

For $b_1,\dots,b_n,a_1\stackrel{iid}{\sim} \dmix$, we need the following lemma. 

\begin{lemma}[Proof in \cref{sec:proof-bound-two-expectations}]\label{lem:bound-two-expectations}
Assume $d$, $n>d+2$ and $P$ are fixed, where $P\in \bR^{n\times n}$ is an orthogonal projection matrix whose rank is $d$. Define $z\triangleq \frac{b^\top (I-P)b}{\|b\|^2}$, where $b=[b_1,\dots,b_n]^\top \in \bR^n$. 
If $a_1, \ b_1, \cdots , \ b_n \stackrel{iid}{\sim} \dmix$, we have
$\bE[1/z]\le \frac{n-2+\sqrt{d}}{n-d-2}$ and $\bE[a_1^2/ b^\top(I-P)b]\le \frac{2/(3\sigma^2)+1}{n-d-2} $.
\end{lemma}
\cref{lem:bound-two-expectations} implies that \[
\bE_{b,a_1}[1/z]\le \frac{n-2+\sqrt{d}}{n-d-2}\,,\quad 
\bE_{b,a_1}\left[\frac{a_1^2}{b^\top (I-P)b}\right]< \frac{2/(3\sigma^2)+1}{n-d-2}\,.
\]
Therefore, we conclude that \[\bE_{b,a_1}\left\|  \begin{bmatrix}
A^\top \\
b^\top
\end{bmatrix}^+ \begin{bmatrix}
x\\
a_1
\end{bmatrix}\right\|^2 \le \frac{(n-2+\sqrt{d})\|v\|^2+ 2/(3\sigma^2)+1}{n-d-2}
\,.
\]

\subsection{Proof of \cref{lem:bound-two-expectations}}\label{sec:proof-bound-two-expectations}
\cref{lem:stochastic-dominate} shows that a noncentral $\chi^2$ distribution first-order stochastically dominates a central $\chi^2$ distribution of the same degree of freedom. It will be needed in the proof of \cref{lem:bound-two-expectations}.
\begin{lemma}\label{lem:stochastic-dominate}

Assume that random variables $X\sim \chi^2(k,\lambda)$ and $Y\sim \chi^2(k)$, where $\lambda>0$. For any $c>0$, we have $$\bP(X\ge c)>\bP(Y\ge c).$$ In other words, the random variable $X$ (first-order) stochastically dominates $Y$. 
\end{lemma}
\begin{proof}
Let $Y_1,X_2,\dots,X_k\stackrel{iid}{\sim}\cN(0,1)$ and $X_1\sim \cN(\sqrt{\lambda},1)$ and all these random variables are jointly independent. Then $X'\triangleq \sum_{i=1}^k X_i^2\sim \chi^2(k,\lambda)$ and $Y'\triangleq Y_1^2+\sum_{i=2}^k X_i^2\sim \chi^2(k)$.

It suffices to show that $\mathbb{P}(X' \geq c) > \mathbb{P}(Y' \geq c)$, or equivalently, $\mathbb{P}(|\cN(\mu,1)|\geq c) > \mathbb{P}(|\cN(0,1)| \geq c)$ for all $c >0$ and $\mu\triangleq \sqrt{\lambda} >0$. Denote $F_c(t) = \mathbb{P}(|\cN(\mu,1)|\geq c)$ and we have 
\begin{align*}
    F_c (\mu) &=  1 - \frac{1}{\sqrt{2 \pi}} \int_{-c}^{c} \exp \left( - \frac{(x-\mu)^2}{2} \right) \ dx = 1 - \frac{1}{\sqrt{2 \pi}} \int_{-c-\mu}^{c-\mu} \exp \left( - \frac{x^2}{2} \right) \ dx,
\end{align*} and thus 
\begin{align*}
    \frac{d F_c(\mu)}{d\mu} = \frac{1}{\sqrt{2 \pi}} \left[ \exp \left( - \frac{(c-\mu)^2}{2}\right) - \exp \left( - \frac{(c+\mu)^2}{2}\right)  \right]> 0.
\end{align*}This shows $\mathbb{P}(|\cN(\mu,1)|\geq c) > \mathbb{P}(|\cN(0,1)| \geq c)$ and we are done.

\end{proof}

\begin{proof}[Proof of \cref{lem:bound-two-expectations}]

Since $b_i\stackrel{iid}{\sim} \dmix$, we can rewrite $b= u + w$ where $w \sim \cN(0 , \sigma^2 I_n)$ and the entries of $u$ satisfy $u_i \stackrel{iid}{\sim} \unif(\{ -1,0,1\})$. Furthermore, $u$ and $w$ are independent. Similarly, we can write $a_1 = \hat{u}+\hat{w}$, where $\hat{u}\sim \unif(\{ -1,0,1\})$ and $\hat{w}\sim \cN(0,\sigma^2)$ are independent. To bound $\bE[a_1^2]$, we have \[
\bE[a_1^2] = \bE[(\hat{u}+\hat{w})^2] = \bE[\hat{u}^2]+\bE[\hat{w}^2] = \frac{2}{3} + \sigma^2\,.
\]

Note that
\[
\frac{1}{z} = \frac{b^\top I b}{b^\top (I-P) b} = 1 + \frac{(u+w)^\top P (u+w)}{(u+w)^\top (I-P)(u+w)}.
\]Since $P$ is an orthogonal projection, there exists an orthogonal transformation $O$ depending only on $P$ such that $$(u+w)^\top P(u+w) = [O(u+w)]^\top D_d [O(u+w)]$$ where $D_d = \diag([1,\dots, 1 , 0 \dots, 0])$ with $d$ diagonal entries equal to 1 and the others equal to 0. We denote $\Tilde{u} = O(u)$, which is fixed (as $u$ and $O$ are fixed), and $\Tilde{w} = O(w) \sim \cN(0, \sigma^2 I_n)$. 
It follows that 
\[
\frac{1}{z} = 1 + \frac{ (\Tilde{u} + \Tilde{w})^\top D_d (\Tilde{u} + \Tilde{w}) }{(\Tilde{u} + \Tilde{w})^\top (I - D_d) (\Tilde{u} + \Tilde{w}) } 
= 1 + \frac{\sum_{i=1}^d (\Tilde{u}_i + \Tilde{w}_i)^2}{\sum_{i=d+1}^n (\Tilde{u}_i + \Tilde{w}_i)^2} 
= 1+ \frac{\sum_{i=1}^d (\Tilde{u}_i + \Tilde{w}_i)^2/\sigma^2}{\sum_{i=d+1}^n (\Tilde{u}_i + \Tilde{w}_i)^2/\sigma^2}\,.
\]
Observe that \begin{align*}
    \sum_{i=1}^d (\Tilde{u}_i + \Tilde{w}_i)^2/\sigma^2 \sim{}& \chi^2\left(d,\sqrt{\sum_{i=1}^d \Tilde{u}^2_i}\right)\\
    \sum_{i=d+1}^n (\Tilde{u}_i + \Tilde{w}_i)^2/\sigma^2 \sim{}& \chi^2\left(n-d,\sqrt{\sum_{i=d+1}^n \Tilde{u}^2_i}\right)\,,
\end{align*}
and that these two quantities are independent. 
It follows that \[
\bE\left[ \sum_{i=1}^d (\Tilde{u}_i + \Tilde{w}_i)^2/\sigma^2\middle|u \right]
= d+\sqrt{\sum_{i=1}^d \Tilde{u}^2_i}\,.
\]
By \cref{lem:stochastic-dominate}, the denominator $\sum_{i=d+1}^n (\Tilde{u}_i + \Tilde{w}_i)^2/\sigma^2$  first-order stochastically dominates $\chi^2(n-d)$. Therefore, we have \[
\bE\left[\frac{1}{\sum_{i=d+1}^n (\Tilde{u}_i + \Tilde{w}_i)^2/\sigma^2}\middle| u\right]\le \bE\left[\frac{1}{\chi^2(n-d)}\right] = \frac{1}{n-d-2}\,.
\]
Putting the numerator and denominator together yields \[
\bE\left[\frac{1}{z}\middle| u\right] \le 1 + \frac{d+\sqrt{\sum_{i=1}^d \Tilde{u}^2_i}}{n-d-2} \le 1+\frac{d+\sqrt{d}}{n-d-2}
= \frac{n-2+\sqrt{d}}{n-d-2}
\,.
\]

Similarly, we have
\begin{align*}
\bE\left[\frac{1}{b^\top (I-P) b}\middle|u\right] ={}& \bE\left[\frac{1}{[O(u+w)]^\top (I - D_d) [O(u+w)]}\middle|u\right]\\
={}&\bE\left[\frac{1/\sigma^2}{\sum_{i=d+1}^n (\Tilde{u}_i + \Tilde{w}_i)^2/\sigma^2}\middle|u\right]\\
\le{}& \frac{1}{\sigma^2} \bE\left[\frac{1}{\chi^2(n-d)}\right]\\
={}& \frac{1}{\sigma^2} \cdot \frac{1}{n-d-2}  
\,.
\end{align*}
Thus, we obtain \[
\bE[1/z]\le \frac{n-2+\sqrt{d}}{n-d-2}\,,\quad \bE\left[\frac{1}{b^\top (I-P) b}\right] \le \frac{1}{\sigma^2} \cdot \frac{1}{n-d-2} \,.
\]
It follows that \[
\bE\left[\frac{a_1^2}{b^\top (I-P) b}\right] \le \frac{2/3+\sigma^2}{\sigma^2} \cdot \frac{1}{n-d-2}
=  \frac{2/(3\sigma^2)+1}{n-d-2}\,.
\]

\end{proof}

\subsection{Proof of \cref{thm:ascent}}\label{sec:proof-ascent}

We start from \eqref{eq:three-terms}. Taking expectation over all random variables gives \begin{align*}
   & \bE\left[\left\|  \begin{bmatrix}
A^\top \\
b^\top
\end{bmatrix}^+ \begin{bmatrix}
x\\
a_1
\end{bmatrix}\right\|^2 -  \left\|  (A^+)^\top x \right\|^2\right]\\
={}& \bE\left[ \|(I-Q)(I+PQ/z)(A^+)^\top x\|^2 - \|(A^+)^\top x\|^2 + \frac{a_1^2}{b^\top(I-P)b} \right]\\
 \ge{}&   -\bE\|(A^+)^\top x\|^2 +
\bE\left[\frac{a_1^2}{\sum_{i=1}^n b_i^2}\right]
 \,.
\end{align*}
Our strategy is to choose $\sigma$ so that $\bE\left[\frac{a_1^2}{\sum_{i=1}^n b_i^2}\right]$ is sufficiently large. This is indeed possible as we immediately show. 
Define independent random variables $u\sim \unif(\{-1,0,1\})$ and $w\sim \cN(0,\sigma^2)$. 
Since $a_1$ has the same distribution as $u+w$, we have \[
\bE[a_1^2] = \bE[(u+w)^2] = \bE[u^2] + \bE[w^2] \ge \frac{2}{3}\,.
\]
On the other hand, 
\begin{align*}
    \bE \left[ \frac{1}{\sum_{i=1}^n b_i^2} \right] & \geq \mathbb{P} ( \max_i |b_i| \leq \sigma  ) \ \bE \left[\frac{1}{\sum_{i=1}^n b_i^2}  \middle| \max_i |b_i| \leq \sigma \right] \\ 
    & = \left[ \mathbb{P} (|b_1| \leq \sigma)\right]^n \ \bE \left[\frac{1}{\sum_{i=1}^n b_i^2}  \middle| \max_i |b_i| \leq \sigma\right] \\
    & \geq \left[ \frac{1}{3 \sqrt{2 \pi \sigma^2}} \int_{- \sigma}^{\sigma}\exp\left( - \frac{t^2}{2 \sigma^2} \right) \ dt \right]^n \frac{1}{n \sigma^2} \\
    & \geq \frac{1}{5^n n \sigma^2}\,.
\end{align*}
Together we have \[\bE\left[\frac{a_1^2}{\sum_{i=1}^n b_i^2}\right] \geq \frac{1}{5^{n+1}n\sigma^2}\,.\] 
As a result, we conclude  \begin{equation*}
    \lim_{\sigma\to 0^+} \bE\left[\left\|  \begin{bmatrix}
A^\top \\
b^\top
\end{bmatrix}^+ \begin{bmatrix}
x\\
a_1
\end{bmatrix}\right\|^2 -  \left\|  (A^+)^\top x \right\|^2\right]
= +\infty\,,
\end{equation*}
which completes the proof.

\section{Proofs for Overparametrized Regime}
\subsection{Proof of \cref{lem:pseudo-inverse-overparam}}\label{sec:proof-pseudo-inverse-overparam}

    Since $A$ and $B$ have full row rank, $(AA^\top)^{-1}$ and $(BB^\top)^{-1}$ exist. Therefore we have $$B^+ = B^\top (BB^\top)^{-1}.$$ The Sherman-Morrison formula gives \[
    (BB^\top)^{-1} = (AA^\top + bb^\top)^{-1} = G - \frac{Gbb^\top G}{1+b^\top G b}
    = G - Gbu = G(I-bu)\,.
    \]
    Hence, we deduce \[
    B^+ = [A,b]^\top G(I-bu) = \begin{bmatrix}
                A^\top G(I-bu)\\
                b^\top G(I-bu)
    \end{bmatrix}
    = \begin{bmatrix}
                A^+ (I-bu)\\
                b^\top G(I-bu)
    \end{bmatrix}
    = \begin{bmatrix}
                A^+ (I-bu)\\
                u
    \end{bmatrix}\,.
    \]
    Transposing the above equation yields to the promised equation. 

\subsection{Proof of \cref{lem:overparam-assumptions-hold}}\label{sec:proof-overparam-assumptions-hold}
Let us first denote $$v \triangleq (A^+)^\top x$$ and $$G \triangleq (AA^\top)^{-1}\in \bR^{n\times n}.$$ First note that by Cauchy-Schwarz inequality, it suffices to show there exists $\cD$ such that $\bE[\lambda^4_\textnormal{max}(G)] < +\infty$ and $\bE\|v\|^4 < +\infty$.  

We define $A_d\in \bR^{n\times d}$ to be the submatrix of $A$ that consists of all $n$ rows and first $d$ columns. 
Denote $$G_d \triangleq (A_d A_d^\top)^{-1}\in \bR^{n\times n}.$$ We will prove $\bE[\lambda_{\textnormal{max}}^4(G)] < +\infty$ by induction. 

The base step is $d=n+8$. Recall $\cD_{[1:n+8]} = \cN(0, I_{n+8})$. 
We first show $\bE [\lambda_\textnormal{max}(G_{n+8})]^4 < +\infty$. Note that since $G_{n+8}$ is almost surely positive definite, 
\begin{align*}
    \bE [\lambda_\textnormal{max}^4 (G_{n+8})] = \bE [\lambda_\textnormal{max}(G_{n+8}^4)] \leq \bE \tr(G_{n+8}^4) = \bE \tr((A_{n+8} A_{n+8}^\top)^{-4}) =  \tr ( \bE[(A_{n+8} A_{n+8}^\top)^{-4}] )\,. 
\end{align*}

By our choice of $\cD_{[1:n+8]}$, the matrix $(A_{n+8} A_{n+8}^\top)^{-1}$ is an inverse Wishart matrix of size $n \times n$ with $(n+8)$ degrees of freedom, and thus has finite fourth moment (see, for example, Theorem 4.1 in \citep{von1988moments}). It then follows that $$\bE [\lambda_\textnormal{max}^4(G_{n+8})]  \leq \tr ( \bE[(A_{n+8} A_{n+8}^\top)^{-4}] ) < +\infty\,.$$

For the inductive step, 
assume $\bE [\lambda_\textnormal{max}(G_d)]^4 < +\infty$ for some $d\ge n+8$. We claim that $$\lambda_\textnormal{max}(G_{d+1}) \leq \lambda_\textnormal{max}(G_d)\,,$$ or equivalently, $$\lambda_\textnormal{min}( A_{d}A_{d}^\top ) \le \lambda_\textnormal{min}( A_{d+1}A_{d+1}^\top )\,.$$ Indeed, this follows from the fact that $$A_{d} A_{d}^\top \preccurlyeq  A_d A_d^\top + bb^\top = A_{d+1} A_{d+1}^\top\,,$$ under the Loewner order, where $b\in \bR^{n\times 1}$ is the $(d+1)$-th column of $A$. Therefore, we have $$\bE [\lambda_\textnormal{max}^4(G_{d+1})] \leq \bE [\lambda_\textnormal{max}^4(G_d)]$$ and by induction, we conclude that $\bE [\lambda_\textnormal{max}^4(G)] < +\infty$ for all $d \geq n+8$.

Now we proceed to show $\bE \|v\|^4 < +\infty$. We have
\begin{align*}
    \|v\|^4 = \| (AA^\top)^{-1} A x \|^4 \leq \| (AA^\top)^{-1} A \|_{op}^4 \cdot \|x\|^4\,,  
\end{align*}where $\|\cdot\|_{op}$ denotes the $\ell^2\to \ell^2$ operator norm. Note that 
\begin{align*}
    \| (AA^\top)^{-1} A \|_{op}^4 & = \lambda^2_{\textnormal{max}} \left( \left((AA^\top)^{-1} A \right)^\top (AA^\top)^{-1} A  \right)
    \\ & = \lambda^2_{\textnormal{max}} \left( A^\top (AA^\top)^{-2} A \right) 
    \\ & = \lambda_{\textnormal{max}} \left( \left( A^\top (AA^\top)^{-2} A \right)^2 \right)\,, 
\end{align*}where the last equality uses the fact that $A^\top (AA^\top)^{-2} A$ is positive semidefinite. Moreover, we deduce
\begin{align*}
    \| (AA^\top)^{-1} A \|_{op}^4 & = \lambda_{\textnormal{max}} \left( A^\top (AA^\top)^{-3} A \right) 
    \\ & \leq \tr \left( A^\top (AA^\top)^{-3} A \right)
    \\ & = \tr \left( (AA^\top)^{-3} A A^\top \right)
    \\ & = \tr \left( (AA^\top)^{-2} \right)\,. 
\end{align*} Using the fact that $A_{d} A_{d}^\top \preccurlyeq A_{d+1} A_{d+1}^\top$ established above, induction gives $$(AA^\top)^{-2} \preccurlyeq (A_{n+8}A_{n+8}^\top)^{-2}. $$ It follows that 
\begin{align}
    \bE \left[ \| (AA^\top)^{-1} A \|_{op}^4 \right] &  \leq \bE \left[ \tr\left( \left(A_{n+8} A_{n+8}^\top\right)^{-2}\right) \right] = \tr \left( \bE \left[ \left(A_{n+8} A_{n+8}^\top\right)^{-2} \right] \right)< + \infty\,, \label{eq:expect-operator-norm}
\end{align} where again we use that fact that inverse Wishart matrix $\left(A_{n+8} A_{n+8}^\top\right)^{-1}$ has finite second moment. 

Next, we demonstrate $\bE \|x\|^4 < +\infty$.
Recall that every $\cD_i$ is either a Gaussian or a Gaussian mixture distribution. Therefore, every entry of $x$ has a subgaussian tail, and thus $\bE \|x\|^4 < + \infty$. Together with \eqref{eq:expect-operator-norm} and the fact that $x$ and $A$ are independent, we conclude that
\[
\bE \|v\|^4 \leq \bE \left[ \| (AA^\top)^{-1} A \|_{op}^4 \right]\cdot \bE \left[\|x\|^4 \right] < + \infty\,.
\]

\subsection{Proof of \cref{thm:descent-overparam2}}\label{sec:proof-descent-overparam2}
The randomness comes from $A,x,a_1$ and $b$. We first condition on $A$ and $x$ being fixed. 

Let $G \triangleq (AA^\top)^{-1}\in \bR^{n\times n}$ and $
    u \triangleq \frac{b^\top G}{1+b^\top G b}\in \bR^{1\times n}$. 
    Define $$v \triangleq (A^+)^\top x\,,\quad r\triangleq 1+b^\top Gb\,,\quad H \triangleq{} bb^\top\,.$$ 
    We compute the left-hand side but take the expectation over only $a_1$ for the moment
    \begin{align}
    & \bE_y\left\|  \begin{bmatrix}
A^\top \\
b^\top
\end{bmatrix}^+ \begin{bmatrix}
x\\
a_1
\end{bmatrix}\right\|^2 - \left\|  (A^+)^\top x \right\|^2 \nonumber\\
={}&  \bE_y\left\| (I-bu)^\top v + u^\top a_1 \right\|^2 -  \|v\|^2\nonumber\\
={}&  \|(I-bu)^\top v\|^2 + \bE_y \|u^\top a_1\|^2 -  \|v\|^2 \nonumber \tag*{($\bE[a_1]=0$)}\\
={}&  \|(I-bu)^\top v\|^2 + \bE_y [a_1^2]\frac{\|Gb\|^2}{r^2} -  \|v\|^2 \,.\nonumber
\end{align}

Let us first consider the first and third terms of the above equation:
\begin{align*}
\|(I-bu)^\top v\|^2 - \|v\|^2 \nonumber
={}& v^\top\left( (I-bu)(I-bu)^\top - I \right)v\nonumber\\
={}& -v^\top \left( bu + u^\top b^\top - b u u^\top b^\top\right)v\nonumber\\
={}& -v^\top\left( \frac{HG + GH}{r} - \frac{H G^2 H}{r^2}\right)v\,.
\end{align*}
Write $G = V\Lambda V^\T$, where $\Lambda = \diag(\lambda_1,\dots,\lambda_n)\in \bR^{n\times n}$ is a diagonal matrix ($\lambda_i > 0$) and $V\in \bR^{n\times n}$ is an orthogonal matrix. Recall $b\sim \cN(0,\sigma^2 I_n)$. Therefore $w \triangleq V^\T b\sim \cN(0,\sigma^2 I_n)$. Taking the expectation over $b$, we have \begin{equation*}
    \bE_b\left[ \frac{HG+GH}{r} \right] =  \bE_b\left[ V\frac{V^\T bb^\T V\Lambda  + \Lambda V^\T bb^\T V}{1+b^\T V\Lambda V^\T b} V^\T \right]
    = V \bE_w \left[ \frac{ww^\T \Lambda + \Lambda ww^\T}{1+w^\T \Lambda w}  \right] V^\T\,.
\end{equation*}
Let $R \triangleq \bE_w \left[ \frac{ww^\T \Lambda + \Lambda ww^\T}{1+w^\T \Lambda w}  \right]$. We have \[
 R_{ii} = \bE_w\left[  \frac{2\lambda_i w_i^2}{1+ \sum_{i=1}^n \lambda_i w_i^2}\right] 
 = \sigma^2 \bE_{\nu\sim \cN(0,I_n)}\left[ \frac{2\lambda_i \nu_i^2}{1+ \sigma^2\sum_{i=1}^n \lambda_i \nu_i^2} \right] >0
\] and if $i\ne j$, \[
R_{ij} = \bE_w \left[ \frac{(\lambda_i +\lambda_j) w_i w_j}{1+ \sum_{i=1}^n \lambda_i w_i^2} \right] \,.
\]
Notice that for any $w$ and $j$, it has the same distribution if we replace $w_j$ by $-w_j$. As a result, \[
R_{ij} = \bE_w \left[ \frac{(\lambda_i +\lambda_j) w_i (-w_j)}{1+ \sum_{i=1}^n \lambda_i w_i^2} \right] = -R_{ij}\,.
\]
Thus the matrix $R$ is a diagonal matrix and \[
R =  2\sigma^2 \frac{\Lambda \diag(\nu)^2}{1+\sigma^2 \nu^\T \Lambda \nu}\,. 
\]
Thus we get \[
\bE_{b,A}\left[ \frac{HG+GH}{r} \right] = 2\sigma^2 \bE_{\nu\sim \cN(0,I_n), A}\left[  \frac{G V \diag(\nu)^2 V^\T}{1+\sigma^2 \nu^\T \Lambda \nu} \right]
\]
Moreover, by the monotone convergence theorem, we deduce \begin{align*}
\lim_{\sigma\to 0^+} \bE_{\nu\sim \cN(0,I_n),A,x}\left[ - v^\T \frac{G V \diag(\nu)^2 V^\T}{1+\sigma^2 \nu^\T \Lambda \nu} v \right] ={}& \bE_{\nu\sim \cN(0,I_n),A,x} \left[ -v^\T G V\diag(\nu)^2 V^\T v \right]\\
={}& \bE[-v^\T G v] \,.
\end{align*}
It follows that as $\sigma\to 0^+$, \[
\bE\left[ -v^\T \frac{HG+GH}{r} v \right] \sim -2\sigma^2 \bE[v^\T G v] = -2\sigma^2 
\bE\left[
v^\T (AA^\T)^{-1} v
\right] = -2\sigma^2 \bE[\|(A^\T A)^+ x\|^2] \,.
\]
Moreover, by \eqref{eq:finite-expectations}, we have \[
\bE\left[v^\T (AA^\T)^{-1} v\right] \le \bE\left[\lambda_{\textnormal{max}}\left((AA^\T)^{-1}\right) \|(A^+)^\T x \|^2\right] < +\infty\,.
\]

Next, we study the term $HG^2 H/r^2$:
\begin{align*}
    \bE_{b,A}\left[ \frac{H G^2 H}{r^2} \right] ={}& \bE_{b,A}\left[V \frac{V^\T b b^\T V \Lambda^2 V^\T b b^\T V}{(1+b^\T V \Lambda V^\T b)^2} V^\T \right]\\
={}&   \bE_{w\sim \cN(0,\sigma^2 I_n),A}\left[V \frac{ww^\T \Lambda^2 ww^\T }{(1+w^\T \Lambda w)^2} V^\T \right] \\
={}& \sigma^4 \bE_{\nu\sim \cN(0,I_n), A } \left[ V \frac{\nu \nu^\T \Lambda^2 \nu \nu^\T}{(1+\sigma^2 \nu^\T \Lambda \nu)^2} V^\T \right] \,.
\end{align*}
Again, by the monotone convergence theorem, we have \begin{align*}
    & \lim_{\sigma\to 0^+} \bE_{\nu\sim \cN(0,I_n),A,x } \left[ v^\T V \frac{\nu \nu^\T \Lambda^2 \nu \nu^\T}{(1+\sigma^2 \nu^\T \Lambda \nu)^2} V^\T v \right] 
\\
={}& \bE_{\nu\sim \cN(0,I_n),A,x } \left[v^\T V \nu \nu^\T \Lambda^2 \nu \nu^\T V^\T v \right]\\
={}& \bE_{A,x}\left[v^\T V\left(2\Lambda^2 + I_n \sum_{i=1}^n \lambda_i^2\right)V^\T v\right]\\
={}& \bE\left[ v^\T \left( 2G^2 + \tr(G^2) I_n \right) v\right]\,.
\end{align*}
It follows that as $\sigma\to 0^+$, \begin{align*}
   & \bE_{b,A,x}\left[ \frac{H G^2 H}{r^2} \right]\\
    \sim{}& \sigma^4 \bE\left[  v^\T \left( 2G^2 + \tr(G^2) I_n \right) v\right]\\
={}& \sigma^4 \bE\left[ 2\|(AA^\T)^{-1}v\|^2 + \tr((AA^\T)^{-2})\|v\|^2
\right]\,.
\end{align*}
Moreover, by \eqref{eq:finite-expectations}, we have \[
\bE\left[ 2\|(AA^\T)^{-1}v\|^2 + \tr((AA^\T)^{-2})\|v\|^2
\right] \le 
(n+2)\bE\left[\lambda^2_{\textnormal{max}} ((AA^\top)^{-1}) \| (A^+)^\top x \|^2\right] < +\infty\,.
\]

We apply a similar method to the term $\frac{\|Gb\|^2}{r^2}$. We deduce \[
\frac{\|Gb\|^2}{r^2} = \frac{b^\T G^2 b}{(1+b^\T G b)^2}
= \frac{b^\T V\Lambda^2 V^\T b}{(1+b^\T V\Lambda V^\T b)^2}\,.
\]
It follows that \[
\bE\left[\frac{\|Gb\|^2}{r^2}\right] = \bE_{w\sim \cN(0,\sigma^2 I_n),A}\left[\frac{w^\T \Lambda^2 w}{(1+w^\T \Lambda w)^2}\right]
= \sigma^2\bE_{\nu\sim \cN(0, I_n),A}\left[\frac{\nu^\T \Lambda^2 \nu}{(1+\sigma^2 \nu^\T \Lambda \nu)^2}\right]
\]
The monotone convergence theorem implies \[
\lim_{\sigma\to 0^+} \bE_{\nu\sim \cN(0, I_n),A}\left[\frac{\nu^\T \Lambda^2 \nu}{(1+\sigma^2 \nu^\T \Lambda \nu)^2}\right] = \bE[\nu^\T \Lambda^2 \nu] = \bE[\tr(G^2)]\,.
\]
Thus we get as $\sigma\to 0^+$ \[
\bE_y [a_1^2]\frac{\|Gb\|^2}{r^2} \sim \sigma^4 \bE[\tr(G^2)]\,,
\]
where $\bE[\tr(G^2)] \le n\bE[ \lambda^2_{\textnormal{max}} ((AA^\top)^{-1}) ] <{} + \infty$. 

Putting all three terms together, we have as $\sigma\to 0^+$ \[
L_{d+1} - L_d
\sim  -2\sigma^2 \bE[\|(A^\T A)^+ x\|^2] \,.
\]
Therefore, there exists $\sigma>0$ such that $L_{d+1} - L_d < 0$.

\subsection{Proof of \cref{thm:ascent-overparam}}\label{sec:proof-ascent-overparam}

Again we first condition on $A$ and $x$ being fixed.
Let $G \triangleq (AA^\top)^{-1}\in \bR^{n\times n}$ and $
    u \triangleq \frac{b^\top G}{1+b^\top G b}\in \bR^{1\times n}$
    as defined in \cref{lem:pseudo-inverse-overparam}. 
    We also define the following variables: $$v \triangleq (A^+)^\top x\,,\quad r\triangleq 1+b^\top Gb.$$ 
    
We compute $L_{d+1}-L_d$ but take the expectation over only $a_1$ for the moment
    \begin{align}
    & \bE_y\left\| \begin{bmatrix}
A^\top \\
b^\top
\end{bmatrix}^+ \begin{bmatrix}
x\\
a_1
\end{bmatrix}\right\|^2 - \left\| (A^+)^\top x \right\|^2 \nonumber\\
={}&  \bE_y\left\| (I-bu)^\top v + u^\top a_1 \right\|^2 -  \|v\|^2 \nonumber\\
={}&   \|(I-bu)^\top v\|^2 + \bE_y \|u^\top a_1\|^2 - \|v\|^2  \nonumber \tag*{($\bE[a_1]=0$)}\\
={}&  \|(I-bu)^\top v\|^2 + \bE_y [a_1^2]\frac{\|Gb\|^2}{r^2} - \|v\|^2  \,.
\label{eq:overparam-expect-y-step}
    \end{align}
    
    Our strategy is to make $\bE [a_1^2\frac{\|Gb\|^2}{r^2}]$ arbitrarily large. To this end, by the independence of $a_1$ and $b$ we have 
    \[
    \bE_{a_1,b} \left[ a_1^2 \frac{\|Gb\|^2}{r^2} \right] = \bE_y [a_1^2] \bE_b \left[ \frac{\|Gb\|^2}{r^2} \right]\,.
    \] By definition of $\dmixN$, with probability $2/3$, $a_1$ is sampled from either $\cN(\mu,\sigma^2)$ or $\cN(-\mu,\sigma^2)$, which implies $ \bE [a_1^2] \geq \frac{1}{3} \mu^2 $. For each $b_i$, we have $$\mathbb{P} (|b_i|\in [\sigma , 2 \sigma]) \geq \frac{1}{3} \times \frac{1}{4}.$$ Also note that $G$ is positive definite. It follows that 
    \begin{align*}
        \bE_b \left[ \frac{||Gb||^2}{r^2}\right] & = \bE_b \left[  \frac{||Gb||^2}{  (1 + b^\top G b )^2  }   \right] \geq \bE_b \frac{(\lambda_{\textnormal{min}}(G)||b||)^2}{  ( 1 + \lambda_\textnormal{max}(G) ||b||^2 )^2 } 
        \geq \left( \frac{1}{12} \right)^n \frac{\lambda^2_{\textnormal{min}}(G) n \sigma^2}{ \left( 1 + 4\lambda_{\textnormal{max}}(G) n \sigma^2 \right)^2 } \,. 
    \end{align*} Altogether we have
    \[
    \bE_{a_1,b} \left[ a_1^2 \frac{\|Gb\|^2}{r^2} \right] \geq \frac{1}{3\cdot 12^n}  \frac{n \lambda^2_{\textnormal{min}}(G) \mu^2 \sigma^2}{( 1 + 4 n \lambda_{\textnormal{max}}(G) \sigma^2 )^2} \,.
    \] Let $\mu = 1/ \sigma^2$ and we have
    \begin{align*}
        \lim_{\sigma \to 0^{+}} \bE \left[ a_1^2 \frac{\|Gb\|^2}{r^2} \right] &\geq  \lim_{\sigma \to 0^+} \bE_{A,x} \bE_{a_1,b} \left[  \frac{1}{3\cdot 12^n}  \frac{n \lambda^2_{\textnormal{min}}(G) }{\sigma^2 ( 1 + 4 n \lambda_{\textnormal{max}}(G) \sigma^2 )^2} \right] 
        \\ & =  \bE_{A,x} \bE_{a_1,b} \lim_{\sigma \to 0^+} \left[  \frac{1}{3\cdot 12^n}  \frac{n \lambda^2_{\textnormal{min}}(G) }{\sigma^2 ( 1 + 4 n \lambda_{\textnormal{max}}(G) \sigma^2 )^2} \right] 
        \\ & = + \infty \,,
    \end{align*}where we switch the order of expectation and limit using the monotone convergence theorem. Taking full expectation over $A,x,b$ and $a_1$ of \eqref{eq:overparam-expect-y-step} and using the assumption that $\bE \|v\|^2 < +\infty$ we have
    \[
    L_{d+1} - L_d = \bE_{A,x,b} \|(I-bu)^\top v\|^2 + \bE \left[a_1^2 \frac{\|Gb\|^2}{r^2} \right] - \bE_{A,x} \|v\|^2 \to +\infty
    \] as $\sigma \to 0^{+}$. 

\subsection{Proof of \cref{thm:bias}}\label{sec:bias}

If we define $G \triangleq (AA^\top)^{-1}\in \bR^{n\times n}$ and $
    u \triangleq \frac{b^\top G}{1+b^\top G b}\in \bR^{1\times n}$, \cref{lem:pseudo-inverse-overparam} implies \begin{equation*}
        \begin{bmatrix}
            A^\top\\
            b^\top
        \end{bmatrix}^+ =
        \begin{bmatrix}
        (I-bu)^\top (A^+)^\top, u^\top
        \end{bmatrix}\,.
\end{equation*}
It follows that \begin{align*}
        [A, b]^+ [A, b] & = \begin{bmatrix}
        A^+A-  \frac{ww^\top}{r} & \frac{w}{r}
        \\ \frac{w^\top}{r} & 1-\frac{1}{r}
        \end{bmatrix}\,,
\end{align*}
where \[
     w = A^+ b\,,\quad r  = 1 + b^\top Gb\,.
\]

We obtain the expression for $\Eps_{d+1}$:
\begin{align*}
    \Eps_{d+1} & = \left( [x^\top, a_1] \begin{bmatrix}
                A^\top A - \frac{ww^\top}{r} - I & \frac{w}{r}
                \\ \frac{w^\top}{r} & -\frac{1}{r}
    \end{bmatrix} \begin{bmatrix}
                \beta \\ \beta_1
    \end{bmatrix} \right)^2,
    \\ & = \left[ x^\top\left( A^+A - \frac{ww^\top}{r} - I \right)\beta + \frac{yw^\top\beta}{r} + \frac{x^\top w \beta_1}{r} - \frac{a_1\beta_1}{r} \right]^2
    \\ & = \left[ x^\top (A^+ A-I)\beta + \frac{1}{r} \left( -x^\top w w^\top \beta + x^\top w \beta_1 + a_1 w^\top \beta - a_1 \beta_1 \right) \right]^2 . 
\end{align*}
If $a_1,b_1,\dots,b_n\stackrel{iid}{\sim} \dmixN$ or  $a_1,b_1,\dots,b_n\stackrel{iid}{\sim} \cN(0,\sigma^2)$, it holds that $\bE[a_1]=0\in \bR$,   $\bE[x]=0\in \bR^d$, and $\bE[b]=0\in \bR^{n\times 1}$. Therefore we have
\begin{align*}
    \E\left[  x^\top\left( A^+A -   I \right)\beta \frac{1}{r} x^\top w\beta_1 \right] & = \E\left[ \frac{1}{r} x^\top\left( A^+A -   I \right)\beta x^\top w\right] \E \left[ \beta_1 \right] =0,
    \\ \E\left[ x^\top\left( A^+A - I \right)\beta \frac{1}{r} a_1 w^\top \beta \right ] & = \E\left[ x^\top\left( A^+A  - I \right)\beta \frac{1}{r} \E[a_1] w^\top \beta \right ] =0,
    \\ \E\left[ x^\top\left( A^+A - I \right)\beta \frac{1}{r} a_1 \beta_1 \right ] & = \E\left[ x^\top\left( A^+A - I \right)\beta \frac{1}{r} \E[a_1] \beta_1 \right ] = 0.
\end{align*}
It follows that 
\begin{align*}
    \E[\Eps_{d+1}]  ={}& \E\left[ x^\top(A^+A-I)\beta \right]^2 + \E \left[ \frac{1}{r^2} \left( -x^\top w w^\top \beta + x^\top w \beta_1 + a_1 w^\top \beta - a_1 \beta_1 \right)^2 \right]
    \\ &  + \E\left[ \frac{2}{r} x^\top (A^+A-I)\beta(-x^\top ww^\top\beta)  \right],
\end{align*}which then gives
\begin{align*}
   & \E [\Eps_{d+1}] - \E[\Eps_d]\\
     ={}& \E \left[ \frac{1}{r^2} \left( -x^\top w w^\top \beta + x^\top w \beta_1 + a_1 w^\top \beta - a_1 \beta_1 \right)^2 \right] + \E\left[ \frac{2}{r} x^\top (A^+A-I)\beta(-x^\top ww^\top\beta)  \right].
\end{align*}First, we consider the second term $\E\left[ \frac{2}{r} x^\top (A^+A-I)\beta(-x^\top ww^\top\beta)  \right]$. Note that 
\begin{align*}
    & \E\left[ \frac{2}{r} x^\top (A^+A-I)\beta(-x^\top ww^\top\beta)  \right]
    \\ ={} & \E\left[ - \frac{2}{r} x^\top (A^+A-I)\beta \beta^\top ww^\top x \right]
    \\ ={} & \E\left[ \frac{2}{r} x^\top (I-A^+A) \E[\beta \beta^\top] ww^\top x \right]
    \\ ={} & \rho^2 \E\left[  \frac{2}{r} x^\top (I-A^+A) ww^\top x \right],
\end{align*}where the second equality is because $\beta$ is independent from the remaining random variables and the third step is because of  $\beta\sim\cN(0,\rho^2 I)$. Recalling that $w=A^+b$ and $A^+AA^+ = A^+$,  we have
\begin{align*}
    & \E\left[ \frac{2}{r} x^\top (A^+A-I)\beta(-x^\top ww^\top\beta)  \right]
    \\ ={} & \rho^2 \E\left[  \frac{2}{r} x^\top (I-A^+A) A^+bw^\top x \right]
    \\ ={} & \rho^2 \E\left[  \frac{2}{r} x^\top (A^+-A^+AA^+) bw^\top x \right]
    \\ ={} & 0.
\end{align*} Now we consider the first term $ \E \left[ \frac{1}{r^2} \left( -x^\top w w^\top \beta + x^\top w \beta_1 + a_1 w^\top \beta - a_1 \beta_1 \right)^2 \right]$. Note that all the cross terms vanishes since $\E[\beta]=0$ and $\E[\beta_1]=0$. This implies 
\begin{align*}
    & \E \left[ \frac{1}{r^2} \left( -x^\top w w^\top \beta + x^\top w \beta_1 + a_1 w^\top \beta - a_1 \beta_1 \right)^2 \right] 
    \\ ={} & \E \left[ \frac{1}{r^2}\left( (x^\top w w^\top \beta)^2 + (x^\top w \beta_1)^2 + (a_1 w^\top \beta)^2 + (a_1 \beta_1)^2  \right)  \right]
    \\ ={} & \E\left[  \frac{1}{r^2}\left( \tr(xx^\top ww^\top \beta\beta^\top ww^\top) + \beta_1^2(x^\top w w^\top x) + a_1^2\tr(  ww^\top \beta \beta^\top) + a_1^2 \beta_1^2 \right) \right]
    \\ ={} & \E\left[  \frac{1}{r^2}\left( \rho^2\|w\|^2 \tr(xx^\top ww^\top) + \rho^2(x^\top w w^\top x) + a_1^2\rho^2 \|w\|^2 + a_1^2 \rho^2 \right) \right]\\
    ={}& \rho^2\E \left[  \frac{1}{r^2} (\|w\|^2+1)( (x^\T w)^2 + \bE[a_1^2])  \right],
\end{align*}where the third equality is because of $[\beta^\T,\beta_1]^\T\sim \cN(0,\rho^2 I_{d+1})$. From the above calculation one can see that $\E[\Eps_{d+1}]> \E[\Eps_d]$. 

If $a_1,b_1,\dots,b_n\stackrel{iid}{\sim} \dmixN$,  \cref{thm:ascent-overparam} implies that for any $C>0$, there exist $\mu,\sigma$ such that \[
\bE\left\| \begin{bmatrix}
A^\top \\
b^\top
\end{bmatrix}^+ \begin{bmatrix}
x\\
a_1
\end{bmatrix}\right\|^2 - \bE \left\|  (A^+)^\top x \right\|^2 >C\,.
\] Because $\E[\Eps_{d+1}]\ge \E[\Eps_d]$, we obtain that for any $C>0$, there exist $\mu,\sigma$ such that $\lb_{d+1}-\lb_d > C$. 

If $a_1,b_1,\dots,b_n\stackrel{iid}{\sim} \cN(0,\sigma^2)$, we have as $\sigma\to 0$, \[
\E [\Eps_{d+1}] - \E[\Eps_d] = \rho^2 \sigma^2\E \left[  \frac{1}{r^2} (\sigma^2\|A^+\|^2+1)( \|A^{+\T}x\|^2 + 1)  \right]
\sim \rho^2 \sigma^2 \left(  \E\|A^{+\T}x\|^2 + 1  \right)\,.
\]
From the proof of \cref{thm:descent-overparam2}, we know that as $\sigma\to 0^+$\[
\left\|  \begin{bmatrix}
A^\top \\
b^\top
\end{bmatrix}^+ \begin{bmatrix}
x\\
a_1
\end{bmatrix}\right\|^2 - \bE\left\| (A^\T)^+ x\right\|^2 \sim -2\sigma^2 \bE[\|(A^\T A)^+ x\|^2] \,.
\]
If $
\rho \le \eta\sqrt{ \frac{ \bE[\|(A^\T A)^+ x\|^2]}{ \E\|A^{+\T}x\|^2 + 1  }}$,
we have \[
\lb_{d+1} - \lb_d \sim -\sigma^2\left(2\eta^2 \bE[\|(A^\T A)^+ x\|^2]-
\rho^2 \left(  \E\|A^{+\T}x\|^2 + 1  \right)
\right) \le -\sigma^2 \eta^2 \bE[\|(A^\T A)^+ x\|^2]\,.
\]
As a result, there exists $\sigma>0$ such that for all  $\rho \le \eta\sqrt{ \frac{ \bE[\|(A^\T A)^+ x\|^2]}{ \E\|A^{+\T}x\|^2 + 1}} $, we have $\lb_{d+1}<\lb_d$. 

\section{Discussion}
Recently, there has been growing interest in the comparison and connection between deep learning and classical machine learning methods. For example, clustering, a classical unsupervised machine learning method, was adapted to end-to-end training of image data~\cite{caron2018deep,fei2018exponential,fei2018hidden,fei2019achieving,fei2020achieving}. This paper studied the non-monotonic generalization risk curve of overparametrized linear regression. It would be an interesting future work to study the multiple descent phenomenon in other classical machine learning methods and theoretically understand this phenomenon in deep learning. Moreover, when the multiple descent phenomenon arises in different machine learning models, it remains open whether there is any deep reason in common that accounts for it.

\end{document}

%% file: lin_preamble.tex
\usepackage[utf8]{inputenc}
\usepackage{paralist}
\usepackage{amsmath,amssymb,dsfont,amsthm,dsfont,hyperref}
\usepackage{xfrac}
\usepackage[capitalize]{cleveref}
\usepackage{subcaption}
\usepackage{mathtools}

\usepackage{enumitem}

\usepackage{forloop}

\usepackage{thmtools} 
\usepackage{thm-restate}

\newtheorem{theorem}{Theorem}
\newtheorem{lemma}[theorem]{Lemma}

\theoremstyle{definition}

\newtheorem{remark}{Remark}

\newcommand{\defcal}[1]{\expandafter\newcommand\csname 
	c#1\endcsname{{\mathcal{#1}}}}
\newcommand{\defbb}[1]{\expandafter\newcommand\csname 
	b#1\endcsname{{\mathbb{#1}}}}
\newcommand{\defbf}[1]{\expandafter\newcommand\csname 
	bf#1\endcsname{{\mathbf{#1}}}}
\newcounter{calBbCounter}
\forLoop{1}{26}{calBbCounter}{
	\edef\letter{\Alph{calBbCounter}}
	\expandafter\defcal\letter
	\expandafter\defbb\letter
	\expandafter\defbf\letter
}
\forLoop{1}{26}{calBbCounter}{
	\edef\letter{\alph{calBbCounter}}
	\expandafter\defbf\letter
}

\newcommand{\eps}{\varepsilon}
\newcommand{\E}{\mathbb{E}}

\newcommand{\T}{\top}

\DeclareMathOperator{\diag}{diag}

\DeclareMathOperator{\unif}{Unif}

\DeclareMathOperator{\im}{im}
\DeclareMathOperator{\tr}{tr}
\DeclareMathOperator{\rank}{rank}